\documentclass[10pt,journal,compsoc]{IEEEtran}

\usepackage{hyperref}

\usepackage{multirow}
\usepackage{arydshln}
\usepackage{xcolor}
\usepackage[flushleft]{threeparttable}

\usepackage{subfig}
\usepackage{graphicx}

\usepackage{algorithm}
\usepackage{algpseudocode}
\newcommand{\multilinestate}[1]{%
  \parbox[t]{\linewidth}{\raggedright\hangindent=\algorithmicindent\hangafter=1
    \strut#1\strut}}
\algnewcommand{\Inputs}[1]{%
  \State \textbf{Inputs:}
  \State \hspace*{\algorithmicindent}\multilinestate{#1}
}
\algnewcommand{\Outputs}[1]{%
  \State \textbf{Outputs:}
  \Statex \hspace*{\algorithmicindent}\multilinestate{#1}
}
\algnewcommand{\Initialize}[1]{%
  \State \textbf{Initialize:}
  \Statex \hspace*{\algorithmicindent}\multilinestate{#1}
}
\usepackage{float}
\newfloat{algorithm}{t}{lop}

\usepackage{amsthm}
\usepackage{amssymb}
\newtheorem{theorem}{Theorem}
\usepackage{bm}
\newtheorem{prop}{Proposition}

\newcommand{\argmin}{\mathop{\rm argmin}\limits}
\usepackage{cite}

\begin{document}
\title{AAA: Adaptive Aggregation of Arbitrary Online Trackers with Theoretical Performance Guarantee}
\author{Heon~Song,~Daiki~Suehiro,~and~Seiichi~Uchida
\IEEEcompsocitemizethanks{
\IEEEcompsocthanksitem H. Song and D. Suehiro are with the Department of Advanced Information Technology, Kyushu University, Fukuoka, Japan, and with AIP, RIKEN, Tokyo, 103–0027 Japan.\protect\\
E-mail: \{heon.song@human., suehiro@\}ait.kyushu-u.ac.jp
\IEEEcompsocthanksitem S. Uchida is with the Department of Advanced Information Technology, Kyushu University, Fukuoka, Japan.\protect\\
E-mail: uchida@ait.kyushu-u.ac.jp
}
}


\IEEEtitleabstractindextext{%
\begin{abstract}
For visual object tracking, it is difficult to realize an almighty online tracker due to the huge variations of target appearance depending on an image sequence. This paper proposes an online tracking method that adaptively aggregates arbitrary multiple online trackers. The performance of the proposed method is theoretically guaranteed to be comparable to that of the best tracker for any image sequence, although the best expert is unknown during tracking. The experimental study on the large variations of benchmark datasets and aggregated trackers demonstrates that the proposed method can achieve state-of-the-art performance. The code is available at \url{https://github.com/songheony/AAA-journal}.
\end{abstract}

}

\maketitle

\IEEEdisplaynontitleabstractindextext

\IEEEpeerreviewmaketitle

\IEEEraisesectionheading{\section{Introduction}\label{sec:introduction}}
\IEEEPARstart{V}{isual} object tracking (VOT) is a research field of significant interest, and is widely applied in fields such as video surveillance~\cite{Mangawati_2018}, traffic flow monitoring~\cite{Tian_2011} and autonomous driving~\cite{Yurtsever_2020}. Various tracking methods are proposed every year~\cite{Fiaz_2019}, but VOT still involves issues relating to areas such as target appearance change, target motion change, occlusion, camera motion, environment illumination change~\cite{Kristan_2016}. \par
These issues are amplified in {\em online} tracking tasks, where the target location needs to be determined in a frame-by-frame manner. Even when there is no target-like region in the current frame due to heavy appearance changes or occlusion of the target object, it is necessary to determine the target location before consideration of the next frame. Once an erroneous determination is made, it is difficult to recover and track the target object properly again.\par
Fig.~\ref{fig:hard_tracking} shows issues with online tracking. Twelve state-of-the-art trackers are applied to six benchmark datasets, with monitoring based on the ratio of achievement for the best tracker in each dataset. The results indicate the difficulty of realizing an ``almighty'' tracker even for a single dataset. That is, no single tracker based on a specific criterion can handle extended variation of tracking tasks. 
\par
A promising strategy for more robust online tracking is to use multiple trackers~\cite{Avidan, Grabner_2006, Zhang_2014, Wang_2018, Qi_2016, Qi2019HedgingDF}. Aggregation of various trackers with different characteristics can be expected to produce complementary interaction. In particular, the potential to combine and/or select trackers based on their reliability will make the results more robust than with a single one.\par
However, estimating the reliability of each tracker is not straightforward, and reliability should be updated during the image sequence because the target condition (i.e., the appearance of the target and the background) will change frame-by-frame. Fig.~\ref{fig:change_reliable} shows two examples in which the most reliable tracker (i.e., the one that determines the target) for one frame becomes totally unreliable in a later frame.\par

\begin{figure}[!t]
\centering
\includegraphics[width=1\linewidth]{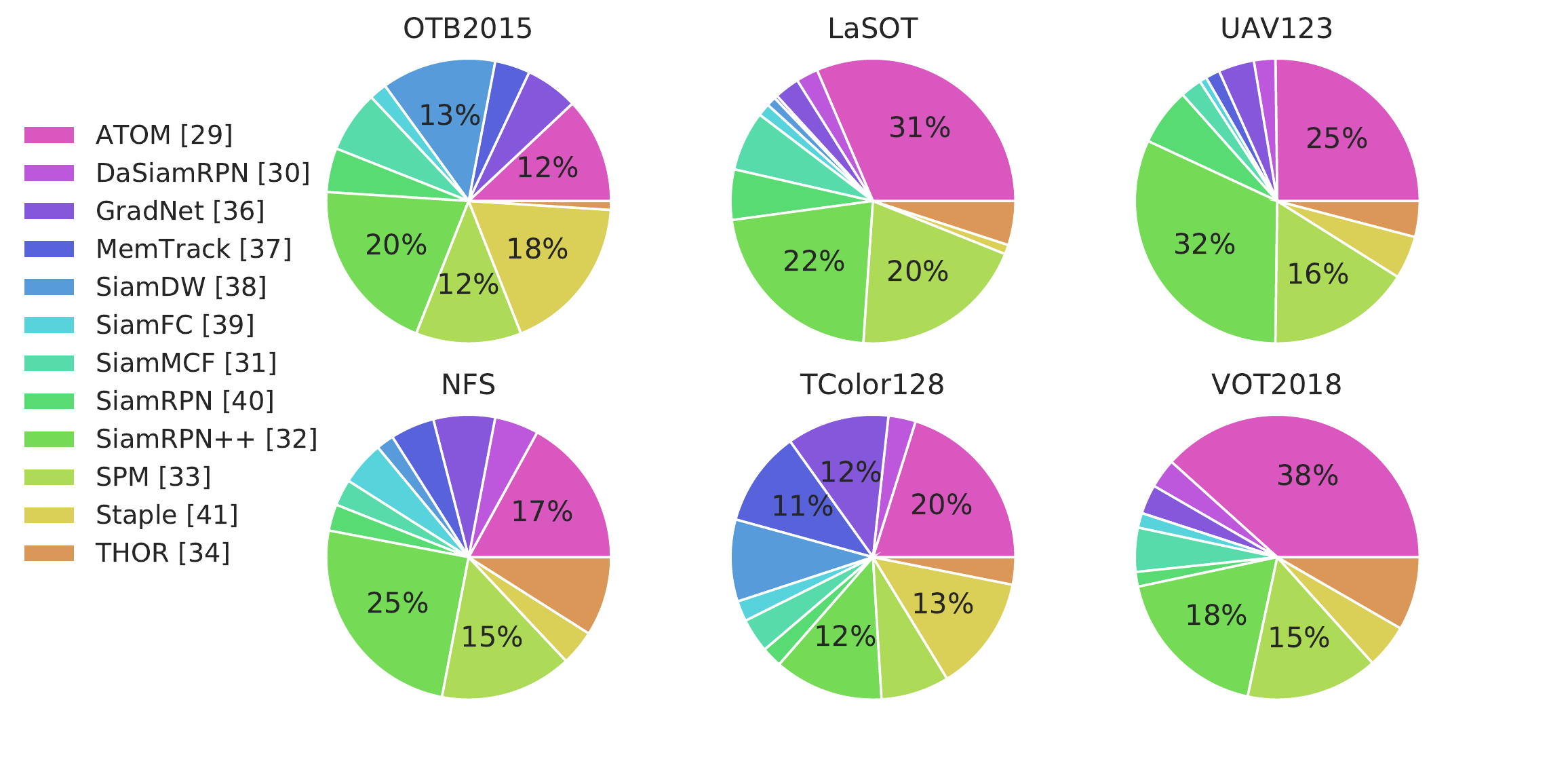}
\caption{{\bf No ``almighty'' tracker.} The percentages represent the ratio of image sequences that each state-of-the-art tracker performs best for the individual tracking benchmarks.}
\label{fig:hard_tracking}
\end{figure}

\begin{figure}[!t]
\centering
\begin{tabular}{c|c}
\includegraphics[width=0.45\linewidth]{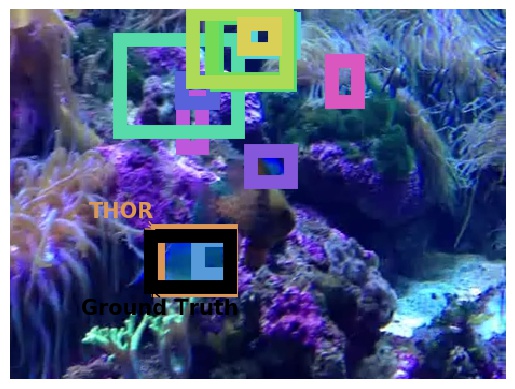} &
\includegraphics[width=0.45\linewidth]{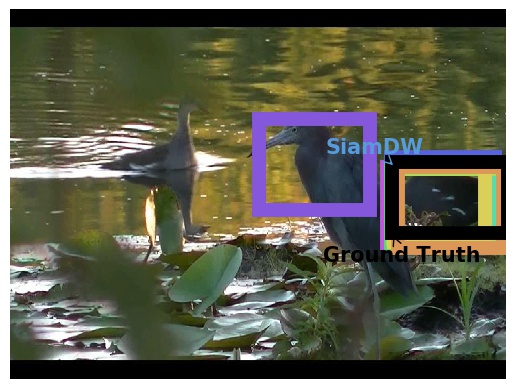} \\
\includegraphics[width=0.45\linewidth]{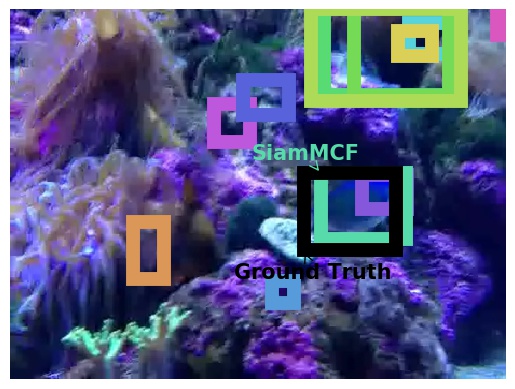} &
\includegraphics[width=0.45\linewidth]{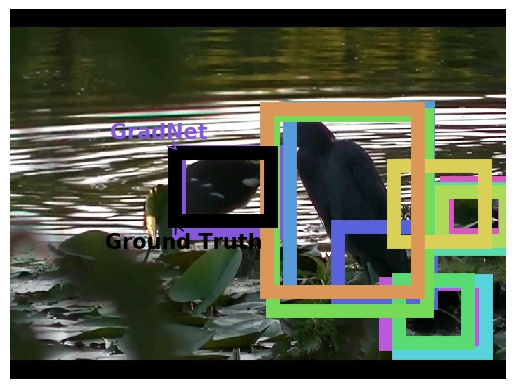}
\end{tabular}
\caption{{\bf The change of the most reliable tracker.} The images of a frame (top) and a few frames after it (bottom) in two image sequences are shown. Each bounding box represents the estimation of experts and the true target location (black). The best tracker changes during the two frames.}
\label{fig:change_reliable}
\end{figure}

In this paper, the authors propose a novel tracking method called Adaptive Aggregation of Arbitrary (AAA) trackers based on {\em Adaptive Expert Aggregation} (AEA). AEA has been studied in the field of theoretical machine learning~\cite{vovk1998game}\footnote{Adaptive expert aggregation is often called \textit{online prediction}~\cite{vovk1998game} or \textit{online learning}~\cite{Shalev_Shwartz_2011} in theoretical machine learning research. In this paper, these terms are avoided in order to avoid confusion with their different meanings in computer vision and pattern recognition research.}, and is a problem involving the aggregation of {\em experts} online.
More specifically, individual {\em experts} give their own solutions to the given task at each time step, and these solutions are then aggregated using a particular algorithm. In AAA, each expert corresponds to an online tracker that estimates the location of the target as its solution, as shown in Fig.~\ref{fig:overview}. $N$ different online trackers will produce $N$ location predictions (i.e., $N$ bounding boxes, often with different sizes) for each frame $t$. These predictions are then aggregated into a single solution (i.e., the target location) for each frame $t$ using a weighted random selection algorithm. 
\par
The strength of AAA is that its performance is theoretically guaranteed in terms of {\em regret} due to the solid theoretical background of AEA. Regret is defined by the difference between the performance of the {\em best expert}\footnote{As noted below, the best expert is unknown during tracking - only at the very end of the image sequence, i.e., at $t=T$, can it be known which $N$ is the best expert for the sequence.} and the performance of the aggregation result.
In VOT with $N$ trackers, the best expert among them gives the best tracking accuracy over an image sequence. If regret is bounded, the accuracy difference between AAA and the best expert can also be bounded. 
It is practically meaningful to have a theoretical guarantee (i.e., a regret bound) because this means that the target location estimated using AAA at frame $t (<T)$ is often not far away from the estimation of the best expert.
\par
This strength of AAA is further emphasized as described here. 
First, this theoretical bound holds with arbitrary experts. Arbitrary trackers (especially state-of-the-art trackers) can therefore be used as experts, whereas the traditional method with multiple trackers often can employ only specific trackers. The bound holds even in an {\em adversarial environment}~\cite{Shalev_Shwartz_2011} in which, for example, 
a tracker that was reliable until $t$ can become totally unreliable at $t+1$. This is not an unrealistic environment, as already observed in Fig.~\ref{fig:change_reliable}. Even for image sequences with such extreme situations, the proposed method is still guaranteed in terms of regret. Even though we can know the best expert for an image sequence may be known only at the last ($T$th) frame, it is still theoretically possible to bound the regret of AAA.\par

\begin{figure}[!t]
\centering
\includegraphics[width=1\linewidth]{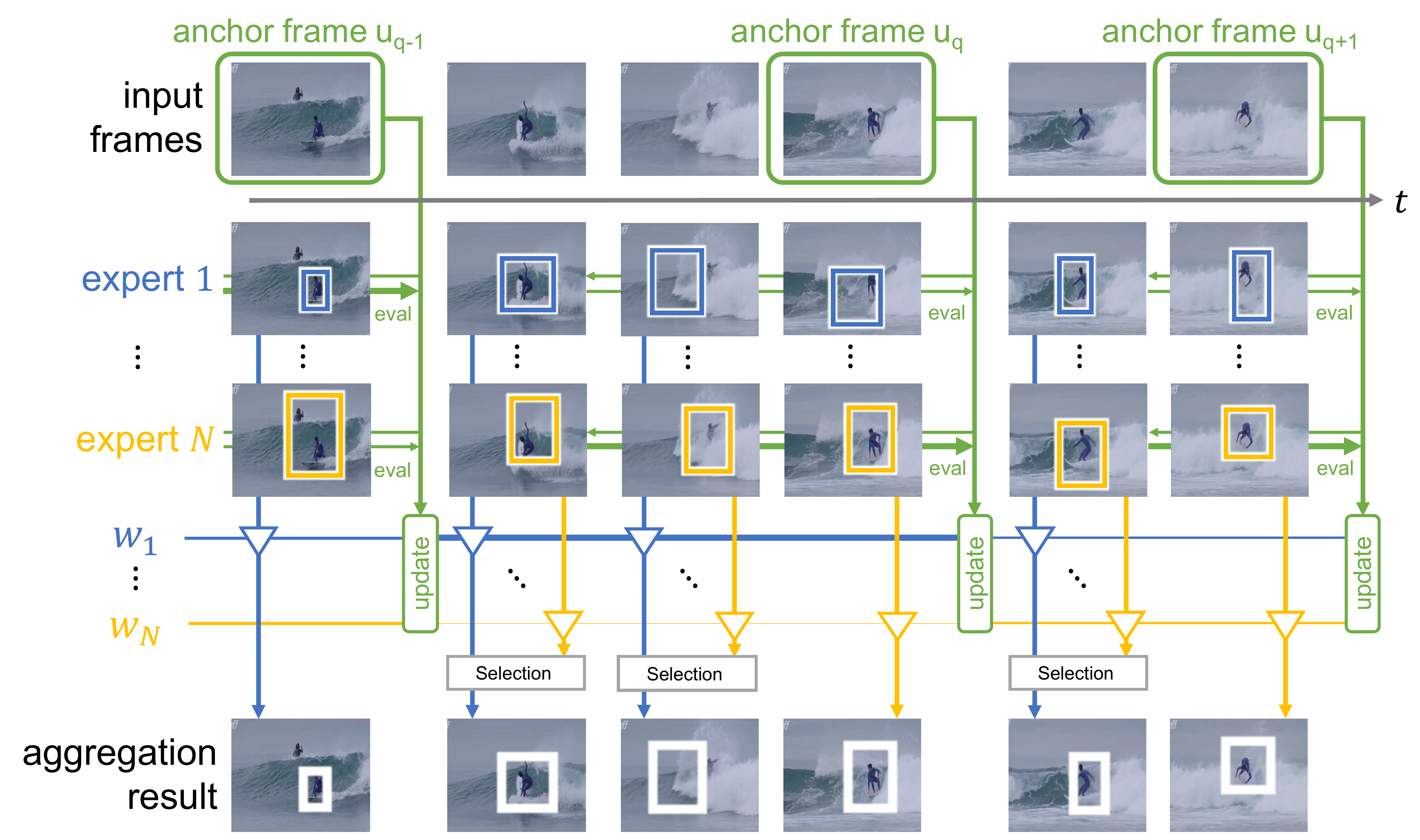}
\caption{The overview of the proposed online tracking method, called Adaptive Aggregation of Arbitrary trackers (AAA). At each frame, the target location is determined as the estimation of an expert which is stochastically selected according to the weight (i.e., reliability) of experts. At the anchor frame, expert weights are updated using delayed feedback.}
\label{fig:overview}
\end{figure}

In AAA, the reliability (i.e., the weight) of each tracker is better evaluated by the proximity of tracker estimation to the true target location (i.e., the ground truth); in theoretical machine learning research, the ground truth for expert evaluation is called {\em feedback}. In VOT tasks, however, it is impossible to obtain exact feedback for each frame because the true target location is not given during online tracking. Accordingly, a practical strategy called  {\em delayed feedback} is adopted in AAA. As shown in Fig.~\ref{fig:overview}, experts receive feedback for {\em anchor frames}, where the target location is determined with high reliability. At the $q$th anchor frame $u_q$, feedback for the $i$th tracker is calculated as the difference from a very reliable {\em offline} tracking result between the previous and current anchor frames $u_{q-1}$ and $u_q$. Since feedback at the frame $\tau\in [u_{q-1}+1,u_q]$ is postponed until $u_q$, this is known as delayed feedback.
\par
It should be noted that the performance of the proposed method is still guaranteed in terms of regret, even with the delayed feedback strategy. Additionally, although offline tracking results are not always exact and thus the delayed feedback is not calculated from the true target location, the theoretical guarantee of the proposed method still holds. These strong theoretical guarantees underpin the very promising performance of the proposed method in various practical situations, as experimentally proved in the later sections.
\par
The main contributions of this paper are as follows:
\begin{itemize}
    \item The authors propose an online tracking algorithm called AAA, by which arbitrary experts (online trackers) are aggregated with promising theoretical performance guarantees.
    \item To the best of the authors' knowledge, this is the first application of an AEA-based algorithm with delayed feedback in a computer vision task.
    \item To demonstrate the experimental performance of the proposed method in an adversarial environment, various experiments were conducted with combinations of arbitrary experts on various datasets.
    \item The experimental results show that the proposed method produced quasi-optimal or optimal performance among state-of-the-art trackers.
\end{itemize}
\par
Numerous extensions are shown in this work from preliminary publication by the authors~\cite{Song2020}. A new important theoretical investigation (Proposition \ref{prop:ours}) is presented to clarify how AAA outperforms other trackers.
The paper outlines more up-to-date experimental validations based on six recent benchmark datasets as well as state-of-the-art online trackers and ensemble tracking methods. The experimental results show that AAA can be used to achieve state-of-the-art performance, and several new experimental setups are added. By way of example, different expert sets were used to allow observation of related effects on performance, with results revealing that AAA is very stable in relation to expert choice.

\section{Related Work\label{sec:related}}
\subsection{Adaptive Expert Aggregation (AEA)}
The goal of AEA is to make predictions with a low regret by aggregating experts. The theories around AEA are discussed in Sec.~\ref{sec:AEA-outline} and~\ref{sec:AEA-regret}, as AEA is not popular in the computer vision field. The Hedge algorithm~\cite{Littlestone_1994, Freund_1997} is a popular algorithm in AEA for prediction of the weighted average of expert solutions. It can be easily applied to expert selection using weights for probability distribution, and achieves good regret bound for the expert selection tasks. \par
One application of the Hedge algorithm involves an adaptive disk spin-down problem. Helmbold et al.~\cite{Helmbold_1996} proposed a method to minimize the energy cost of the disk by aggregating experts estimating the timing of this spin-down. Recently an algorithm called Follow the Regularized Leader (FTRL) has been applied for meta-learning~\cite{finn2017model}. On the other hand, algorithms with delayed feedback have mainly been discussed from a theoretical perspective. Specifically, Quanrud et al.~\cite{Quanrud2015OnlineLW} proposed various algorithms, including the Hedge algorithm with delayed feedback. However, neither practical application nor experimental validation have been completed in ~\cite{Quanrud2015OnlineLW}.

\subsection{Ensemble tracking methods}
Many of the various ensemble tracking methods previously proposed~\cite{Wang2014EnsembleBasedTA, Han2017BranchOutRF, Zhou2014AnEO, Tian2007OnLineES} have leveraged carefully designed synergy among trackers, making it difficult to aggregate arbitrary trackers. By way of example, Avidan et al.~\cite{Avidan} and Grabner et al.~\cite{Grabner_2006} aggregated trackers complementarily trained using AdaBoost~\cite{Freund_1997}. Zhang et al.~\cite{Zhang_2014} proposed a tracking method whose experts are trackers derived from the same (e.g., SVM-based) tracking algorithm. Experts differ in terms of updated frames used to deal with past appearances of the target object. \par
Some ensemble methods can be applied to aggregate arbitrary trackers. Wang et al.~\cite{Wang_2018} proposed a tracking method, called the Multi-Cue Correlation filter based Tracker (MCCT). Similar to the method proposed AAA, any tracker that outputs a bounding box as its prediction can be employed as an expert. However, the practical success of the MCCT is supported by the strong assumption that bounding boxes given by experts are close together. Accordingly, MCCT used specific experts that satisfy assumptions based on inter-expert sharing of ROI.
In aggregation of arbitrary trackers, this assumption may not be met because some trackers may predict a completely different location to others as described in the Appendix~\ref{sec:append_mcct}. The authors' experimental results show that deviation from the assumption degrades MCCT performance.\par
Qi et al.~\cite{Qi2019HedgingDF} proposed an AEA-based tracking method called the Hedged Deep Tracker (HDT*) based on the Hedge algorithm with multiple experts. To the best of the authors' knowledge, this is the only trial to have applied AEA-based aggregation for VOT. As detailed in the Appendix~\ref{sec:append_hdt}, HDT* uses feedback based on expert predictions (rather than other reliable resources such as offline trackers). HDT* also assumes that feedback is given for every frame regardless of reliability. Accordingly, the feedback may be relatively unreliable when a majority of experts do not perform well, and such unreliability degrades overall performance.

\section{Adaptive Aggregation of Arbitrary trackers (AAA)\label{sec:method}}
\subsection{Overview}\label{sec:AAA-overview}
As shown in Fig.~\ref{fig:overview}, the proposed method, AAA, assumes $N$ arbitrary experts (online trackers). At each frame $t$, the method involves stochastic selection of an expert based on weights $w_1^t,\ldots, w_N^t \in \mathbb{R}$ as probabilistic distribution, where $\sum_{i=1}^N w_i^t=1$ and $w_i^t \geq 0$ for any $i$. That is, an expert with a greater weight has a higher chance of being selected. The target location $p^t$ estimated by the selected expert is then used as the output of AAA at $t$. 
Through this simple process, the proposed method enables performance similar to that of the best expert over an image sequence in any situation, even for extreme conditions.
\par
From an algorithmic viewpoint, the main concern is how and when to update the weights $w_1^t,\ldots, w_N^t$. Since $w_i^t$ indicates the reliability of the $i$th tracker, it should be updated using reliable information. Ideally, if the true target location is present at $t$, a greater weight acn be assigned for a tracker that estimates a similar location. However, determination of the true target location is practically impossible. \par
Accordingly, a reliable {\em offline} tracker between two {\em anchor frames} was used to update weights. The frame $t$ is defined as the $q$th anchor frame $u_q$ if the target object is found at a certain location $y^t$ by an object identifier (or tracker) with very high confidence. Connecting two reliable target locations $y^{u_{q-1}}$ and $y^{u_q}$ using an accurate  offline tracker produces the pseudo-ground-truth sequence $y^\tau, \tau \in [u_{q-1}+1, u_q]$. If the target location estimated via the $i$th expert is similar to the pseudo-ground-truth during $[u_{q-1}+1, u_q]$, the weight of the expert will be increased at the anchor frame $u_q$. 
Globally-optimal offline tracking is used based on Dijkstra's algorithm, as detailed in Appendix~\ref{sec:append_offline}.
\par
The pseudo-ground-truth is referred to as feedback based on the terminology of theoretical AEA research because it is used for posterior evaluation in relation to individual experts. Specifically in AAA, it is referred to as delayed feedback because feedback during $\tau\in [u_{q-1}+1, u_q]$ is given in a later frame $u_q$, as shown in Fig.~\ref{fig:overview}. It should be emphasized again that the performance of the proposed method is still guaranteed even with the use of the pseudo-ground-truth as delayed feedback, as detailed in Sec.~\ref{sec:theory}.\par

\begin{algorithm}[t]
\caption{The proposed method: AAA.\label{alg:AAA}}
\begin{algorithmic}
\Inputs{The initial target location $f_0$.\par $N$ arbitrary experts.}
\Outputs{Predicted target location $p^1,\ldots,p^t,\ldots$}
\Initialize{The Weight of the experts $w_i^1 \gets 1/N, \forall i$.\par 
The initial anchor frame $u_1 \gets 1$ and $q \gets 1$.\par
The initial learning rate $\eta \gets \sqrt{\ln N}$.}
\For {$t=2,\dots $}
    \State Get estimation $f^t_1, \ldots, f^t_N$ from the experts.
    \If{$t$ is determined to be an anchor frame}
        \State Increase the number of anchor frame $q \gets q+1$.
        \State Store $t$ as the last anchor frame $u_q \gets t$.
        \State Obtain delayed feedback $y^{u_{q-1}+1}, \ldots, y^{u_q}$.
        \State Calculate each cumulative loss by using (\ref{eq:loss}).
        \State Update $\eta$ using doubling trick.
        \State Update weights by using (\ref{eq:weight}).
        \State Set the target location $p^{t} \gets y^t$.
    \Else
        \State Do not update weights $w_i^{t} \gets w_i^{t-1}, \forall i$
        \State \multilinestate{Select the target location $p^{t}$ from $f^t_1, \ldots, f^t_N$ \newline stochastically using $w_1^t, \ldots, w_N^t$.}
    \EndIf
\EndFor
\end{algorithmic}
\end{algorithm}

\subsection{Updating of expert weights}
Expert weights are updated at each anchor frame using the delayed feedback given at the frame as shown in Fig.~\ref{fig:overview}. 
At the anchor frame $t$, the {\em loss} $L_i$ of the expert $i$ is first calculated using the delayed feedback $y^\tau, \tau\in [u_{q-1}+1, u_q]$ via:
\begin{equation}
\label{eq:loss}
    L_i = \sum_{\tau=u_{q-1}+1}^{u_q} \ell(f_i^\tau, y^\tau), 
\end{equation}
where 
\begin{equation}
\ell(f_i^\tau, y^\tau) = 1 - {\mathcal{P}}(f_i^\tau, y^\tau),
\end{equation}
and $f_i^\tau$ is the $i$th expert's estimation at frame $\tau$. Locations such as $f_i^\tau$ and $y^\tau$ are represented as a bounding box, and the function $\mathcal{P}$ provides evaluation of proximity between two bounding boxes.
Based on this loss, the weight of the expert $i$ is updated at $t=u_q$ via the following equation and used from $t+1$:
\begin{equation}
\label{eq:weight}
    w^{t+1}_i = \frac{w^t_i \exp\left(-\eta L_i\right)}
{\sum_{j=1}^N w_j^t \exp\left(-\eta L_j\right)}.
\end{equation}
In (\ref{eq:weight}), $\eta$ is the learning rate, and its value is carefully and automatically controlled for performance guarantee as detailed in Sec.~\ref{sec:AAA-bound}.
\par
For clarity, AAA is briefly summarized in Algorithm~\ref{alg:AAA}.
The first frame is treated as the first anchor frame $u_1$. The weights are initialized as $w^1_i = 1/N, \forall i$. If the current frame $t$ is not an anchor frame, the weight is not updated, i.e., $w^{t+1}_i=w^t_i$.\par
To evaluate the proximity ${\mathcal{P}}(f_i^\tau,y^\tau)$ of two bounding boxes specified by $f_i^\tau$ and $y^\tau$, IoU~\cite{Wu_2013} is a possible choice. However, if the bounding boxes do not overlap, the IoU score is zero regardless of distance. Accordingly, GIoU~\cite{Rezatofighi_2019} is employed to evaluate both overlap and distance. Any arbitrary function can be used as the loss function $\ell$ to give a theoretical performance guarantee if values are limited in the interval $[0,1]$.

\subsection{Anchor frame determination\label{sec:anchor}}
For accurate delayed feedback close to the exact target location, it is necessary to carefully determine anchor frames because these give the boundary conditions of the offline tracker for such feedback. An anchor frame is determined when the target location is determined with very high confidence. There are several approaches to highly-confident determination. For example, if the target is a specific pedestrian or athlete, the face or jersey number can be used for determination with the help of person re-identification or scene text OCR techniques. \par
For application of AAA to various datasets with various targets, a more general approach can be used for highly-confident target determination with reliance on the target template image. In VOT tasks, the template is generally given as the bounding box at the initial frame. If one or more $N$ experts determine a bounding box whose normalized cosine similarity\footnote{``Normalized'' cosine similarity  $\in [0,1]$ is simply given by $(1 + \mathrm {cosine\ similarity})/2$. } to the template is greater than the threshold $\theta$, the current frame $t$ is determined as an anchor frame. The bounding box of the expert with the maximum similarity is determined as $y^t$. The threshold $\theta$ is the only hyper-parameter of the proposed method. As discussed in Sec.~\ref{sec:threshold}, an appropriate value of $\theta$ is experimentally determined for experts.\par
The cosine similarity between the template and the bounding box determined is evaluated using the feature vectors given by ResNet~\cite{He_2016}.
Specifically, like \cite{Lopes_2017}, the output of the average pooling layer of ResNet is used as a feature vector. Here, ResNet pre-trained with ImageNet~\cite{Russakovsky_2015} is used with no extra training.
Both the template and the bounding box are converted to feature vectors using ResNet, and their normalized cosine similarity is calculated.  

\section{Theoretical guarantee of AAA\label{sec:theory}}
\subsection{General preliminaries of AEA}\label{sec:AEA-outline} 
Before an explanation of the theoretical guarantee of AAA, there is a need for a brief introduction to the general theories surrounding AEA, which is generally considered as a {\em repeated game} between a player and an adversarial environment. At each round $t=1,\ldots,T$, the player receives $N$ advice $f^t_1,\ldots,f^t_N$ from $N$ experts. The player makes a prediction $p^t$ based on this advice, and the environment gives its feedback $y_t$ to $p^t$. The player suffers the loss $\ell\left(p^t, y^t\right)$. \par
Here, AAA is based on AEA and thus has clear correspondence with the above terminologies. Specifically, the round, player and experts correspond to the frame $t$, AAA (i.e., the proposed tracker) and the online trackers to be aggregated, respectively. The prediction $p^t$, advice $f_i^t$, and feedback $y^t$ correspond to the target location determined from AAA, the target locations estimated from the $N$ online trackers, and the offline tracking result, respectively.\par
The goal of AEA is to minimize the regret $R_T$:
\begin{equation}
R_T = \mathbb E\left[\sum_{t=1}^T \ell(p^t, y^t)\right] - \min_{i= 1,\ldots, N} \sum_{t=1}^T  \ell(f_i^t, y^t),
\label{eq:regret}
\end{equation}
where the first and second terms represent the cumulative loss of the player and the {\em best expert}, respectively. The best expert is the one with the minimum cumulative loss among $N$ experts. Intuitively, lower regret means the performance of the player is close to that of the best expert. \par
In AEA with delayed feedback, the player suffers a loss only at $Q (\leq T)$ rounds $u_1, \ldots, u_q, \ldots, u_Q$ rather than at each $t$. At the round $u_q$, the environment gives feedback $y^\tau, \tau\in[u_{q-1}+1, u_q]$ for the predictions of the player between $u_{q-1}$ and $u_q$. In Sec.~\ref{sec:AAA-overview}, it can be seen that $u_q$ corresponds to the $q$th anchor frame in AAA.

\subsection{Regret bound of AEA with delayed feedback}\label{sec:AEA-regret}
The regret bound of AEA with delayed feedback is given as follows:
\begin{theorem}[From Theorem A.5 of \cite{Quanrud2015OnlineLW}]
\label{theo:regret_delay}
Assume an AEA algorithm with the weight-updating strategy of (\ref{eq:weight}) and delayed feedback. Also assume the loss function $\ell\in [0, 1]$ and the learning rate $\eta \propto \sqrt{\ln N/(T+D)}$. The regret of the AEA algorithm after $T$ frames is then bounded as follows:
\begin{equation}
R_T = O \left(\sqrt{\left(T+D\right)\ln N} \right), \label{eq:bound1}
\end{equation}
where $D$ is the total delay.
\end{theorem}
\noindent The {\em total delay} $D$ is defined as $D=\sum_{q=2}^Q\sum_{\tau=1}^{u_q - u_{q-1}}\tau=\sum_{q=2}^{Q}\left(u_q - u_{q-1}\right)\left(u_q - u_{q-1}+1\right)/ 2$.
Thus, $D$ takes its minimum value $T$ when feedback is given at every frame and its maximum value $(T^2+T)/2$ when no feedback is given until $t=T$ after $t=1$ (i.e., $u_1=1, u_2=T$, and $Q=2$). Appendix~\ref{sec:append_regret} details the derivation of Theorem \ref{theo:regret_delay} from Theorem A.5 of \cite{Quanrud2015OnlineLW}.
\par
Theorem~\ref{theo:regret_delay} states that regret $R_T$ increases according to $D$. In the worst case, when $D$ takes its maximum value $(T^2+T)/2$, the regret bound is linearly proportional to $T$. This means that the performance difference between AEA and the best expert will increase drastically with $T$. 

\begin{figure*}[!t]
\centering
\includegraphics[width=1\linewidth]{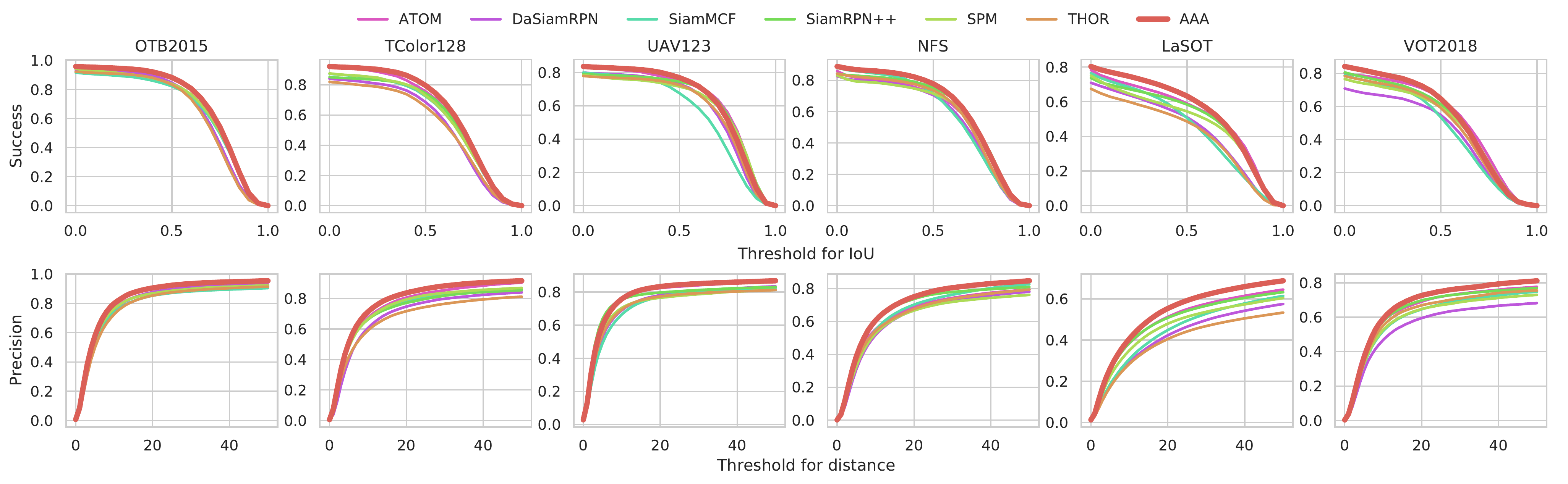}
\caption{Success and precision plots in \textit{High} group.}
\label{fig:high_curve}
\end{figure*}

\begin{table*}[!t]
    \centering
    \begin{threeparttable}
    \caption{Tracking Accuracy by \textit{High} Group\vspace{-2mm}}
    \label{table:high_score}
    \begin{tabular}{c|c|c|c|c|c|c|c|c|c|c|c}
    \hline
    \multirow{2}{*}{Tracker} & \multicolumn{2}{c|}{OTB2015} & \multicolumn{2}{c|}{TColor128} & \multicolumn{2}{c|}{UAV123} & \multicolumn{2}{c|}{NFS} & \multicolumn{2}{c|}{LaSOT} & VOT2018 \\
      & AUC & DP & AUC & DP & AUC & DP & AUC & DP & AUC & DP & AUC \\
    \hline\hline
    ATOM~\cite{Danelljan_2019} & 0.67 & 0.87 & {\color{blue} \textit{0.60}} & {\color{blue} \textit{0.81}} & {\color{red} \textbf{0.62}} & {\color{blue} \textit{0.82}} & 0.58 & 0.69 & {\color{blue} \textit{0.51}} & {\color{blue} \textit{0.51}} & {\color{blue} \textit{0.52}} \\
    DaSiamRPN~\cite{Zhu_2018} & 0.65 & 0.88 & 0.53 & 0.75 & 0.57 & 0.78 & 0.55 & 0.67 & 0.43 & 0.42 & 0.43 \\
    SiamMCF~\cite{Morimitsu2018MultipleCF} & 0.65 & 0.85 & 0.57 & 0.78 & 0.54 & 0.77 & 0.57 & 0.70 & 0.44 & 0.45 & 0.45 \\
    SiamRPN++~\cite{Li_2019} & {\color{blue} \textit{0.69}} & {\color{blue} \textit{0.90}} & 0.58 & 0.77 & 0.60 & 0.80 & 0.60 & {\color{blue} \textit{0.74}} & 0.49 & 0.51 & 0.50 \\
    SPM~\cite{Wang_2019} & 0.67 & 0.87 & 0.58 & 0.79 & 0.59 & 0.77 & 0.57 & 0.67 & 0.47 & 0.48 & 0.48 \\
    THOR~\cite{Sauer2019BMVC} & 0.64 & 0.85 & 0.52 & 0.72 & 0.57 & 0.77 & 0.57 & 0.68 & 0.40 & 0.41 & 0.47 \\
    \hdashline
    HDT*~\cite{Qi2019HedgingDF}& - & 0.83 & - & 0.71 & - & 0.73 & - & 0.58 & - & 0.37 & - \\
    MCCT~\cite{Wang_2018} & 0.64 & 0.83 & 0.53 & 0.72 & 0.58 & 0.76 & 0.57 & 0.69 & 0.42 & 0.44 & 0.40 \\
    Random & 0.66 & 0.87 & 0.56 & 0.77 & 0.58 & 0.79 & 0.57 & 0.69 & 0.46 & 0.46 & 0.48 \\
    Max & 0.68 & 0.89 & 0.56 & 0.77 & 0.58 & 0.78 & {\color{blue} \textit{0.61}} & 0.74 & 0.46 & 0.46 & 0.46 \\
    AAA(Proposed) & {\color{red} \textbf{0.70}} & {\color{red} \textbf{0.91}} & {\color{red} \textbf{0.62}} & {\color{red} \textbf{0.84}} & {\color{blue} \textit{0.62}} & {\color{red} \textbf{0.83}} & {\color{red} \textbf{0.61}} & {\color{red} \textbf{0.75}} & {\color{red} \textbf{0.53}} & {\color{red} \textbf{0.55}} & {\color{red} \textbf{0.52}} \\
    \hline
    \end{tabular}
    \begin{tablenotes}
      \item{$\cdot$} The best tracker is indicated with {\color{red} \textbf{red bold}} and the second is indicated with {\color{blue} \textit{blue italic}}. 
      \item{$\cdot$} AUC: average area-under-curve score, DP: average distance precision. Bigger AUC and DP mean better performance.
      \item{$\cdot$} HDT is only evaluated by DP for a fair comparison. See Appendix~\ref{sec:append_hdt} for the detailed reason.
    \end{tablenotes}
    \end{threeparttable}
\end{table*}

\subsection{Regret bound of AAA\label{sec:AAA-bound}}
Based on Theorem~\ref{theo:regret_delay}, the regret bound of AAA can be derived. This can be lower than (\ref{eq:bound1}) in general because delayed feedback will be far more frequent than the worst case in real-world tracking. Denoting $r$ as the {\em anchor frame ratio}, which is the probability that a frame is determined as an anchor frame, the regret bound outlined below is derived.
\begin{theorem}
\label{theo:regret_cdelay}
Assume the AEA algorithm of Theorem \ref{theo:regret_delay} can have delayed feedback with the anchor ratio $r \in (0,1]$ at each frame.
The expectation of the regret is then upper-bounded as follows:
\begin{equation}
\mathbb E_{r}\left[R_{T}\right] = O \left(\sqrt{{1 \over r}T \ln N} \right).
\label{eq:bound2}
\end{equation}
\end{theorem}
\begin{proof}
The expected delay length $\mathbb E_r[u_q - u_{q-1}]$ is $1/r$
and the expected number of anchor frames is $rT$. Thus, the expectation of the total delay is $\mathbb E_r[D]= \sum_{q=2}^{Q}\left(1/r\right)\left(1/r+1\right)/2 =rT  \left(1/r\right)\left(1/r+1\right)/2  = O(T/r)$. Finally, consideration for the expectation of the regret of $R_T$ with $r$ produces the above bound.
\end{proof}
In contrast to Theorem \ref{theo:regret_delay}, when the regret increases linearly with $T$ in the worst case, Theorem \ref{theo:regret_cdelay} guarantees that the regret of AAA increases in $O(\sqrt{T})$. 
This indicates that the regret of AAA is bounded more tightly, and stable performance of AAA can therefore be expected even for a longer sequence.

\subsubsection{Impact of the regret bound for VOT}
Theorem \ref{theo:regret_cdelay} guarantees that the performance difference between the proposed AAA and the best expert is upper-bounded by (\ref{eq:bound2}). Intuitively speaking, this bound is significant in a number of ways. First, it means that the performance of AAA is not far from that of the best expert. 
Second, this guarantee holds for arbitrary experts, arbitrarily delayed feedback (i.e., arbitrary offline trackers and the determination rule for anchor frames) and arbitrary image sequences.
Third, and most interestingly, AAA may demonstrate performance similar to that of the best expert {\em even though the best expert and its performance is unknown until $T$, i.e., the end of the image sequence}. Expert selection is made at each frame $t$ in a strictly online condition, and it remains unknown which expert will be the best; nevertheless, these theorems guarantee that the performance of AAA will not be far from that of the best expert.

\subsubsection{Effect of the anchor frame ratio \texorpdfstring{$r$}{}}
One might expect that it is better to set $r=1$ (to make all frames the anchor frames), since the regret bound based on (\ref{eq:bound2}) is minimum when $r=1$. However, it must be remembered that the anchor frame should be set to give reliable delayed feedback with a reliable offline tracking result. All the above theories rely on the loss function $\ell$, which treats $y^t$ given by the offline tracking result as a pseudo-ground-truth.
Accordingly, using unreliable feedback eventually produces a choice far from the true ground-truth as the best expert. Finally, AAA tries to follow this false best expert to keep the regret bound. Consequently, for better AAA performance, anchor frames must be carefully determined with higher settings of $\theta$, even though this makes $r$ smaller. This is experimentally demonstrated in Appendix~\ref{sec:append_threshold}.

\subsubsection{Effect of the number of experts  \texorpdfstring{$N$}{}\label{sec:effect_of_N}}
The theorems remove the need for concern over the choice of experts, because the regret bound is simply a logarithmic representation of the number of experts $N$. In other words, even if many experts are employed, the regret bound will increase only slightly. This increase will not be problematic to practical tracking performance. As regret represents the difference from the best expert, using more experts with different characteristics will increases the chances to have the best expert with better performance. Since the difference is bounded by (\ref{eq:bound2}), the presence of a better best expert will enhance AAA performance.

\subsubsection{Learning rate \texorpdfstring{$\eta$}{}}
\label{subsubsec:doubling_trick}
Theorem \ref{theo:regret_cdelay}, as well as Theorem \ref{theo:regret_delay}, assumes that the learning rate $\eta$ should be proportional to $\sqrt{\ln N/(T+D)}$. However, $T$ and $D$ are usually unknown until the end of the sequence in online tracking tasks. Fortunately, the {\em doubling trick}~\cite{Quanrud2015OnlineLW} allows adaptive control of $\eta$ with the regret guarantee in the same order in Theorem~\ref{theo:regret_cdelay}.
Roughly speaking, this trick uses a tentative value $Z$ instead of $T+D$.
Thus, the parameter $\eta$ is initially set as $\eta = \sqrt{\ln N/Z}$.
Then, if the actual value of $T+D$ reaches $Z$ at the current frame $t$,\footnote{Specifically, this condition means that $t+D_t$ is the same as $Z$, where $D_t$ is the total delay until $t$ defined as 
$D_t=\sum_{\tau=1}^{t - u_{\bar{q}}}\tau + \sum_{q=2}^{\bar{q}}\sum_{\tau=1}^{u_q - u_{q-1}}\tau$. $\bar{q}$ is the latest anchor frame by $t$.} the value of $Z$ is doubled and $\eta$ is updated using the new value of $Z$. The proof that the doubling trick still guarantees the regret bound of Theorem \ref{theo:regret_delay} is detailed in ~\cite{mohri2018foundations} and~\cite{Quanrud2015OnlineLW} and the proof for Theorem \ref{theo:regret_cdelay} is trivial from this.

\section{Experiments\label{sec:experiments}}

\begin{figure*}[!t]
\centering
\includegraphics[width=1\linewidth]{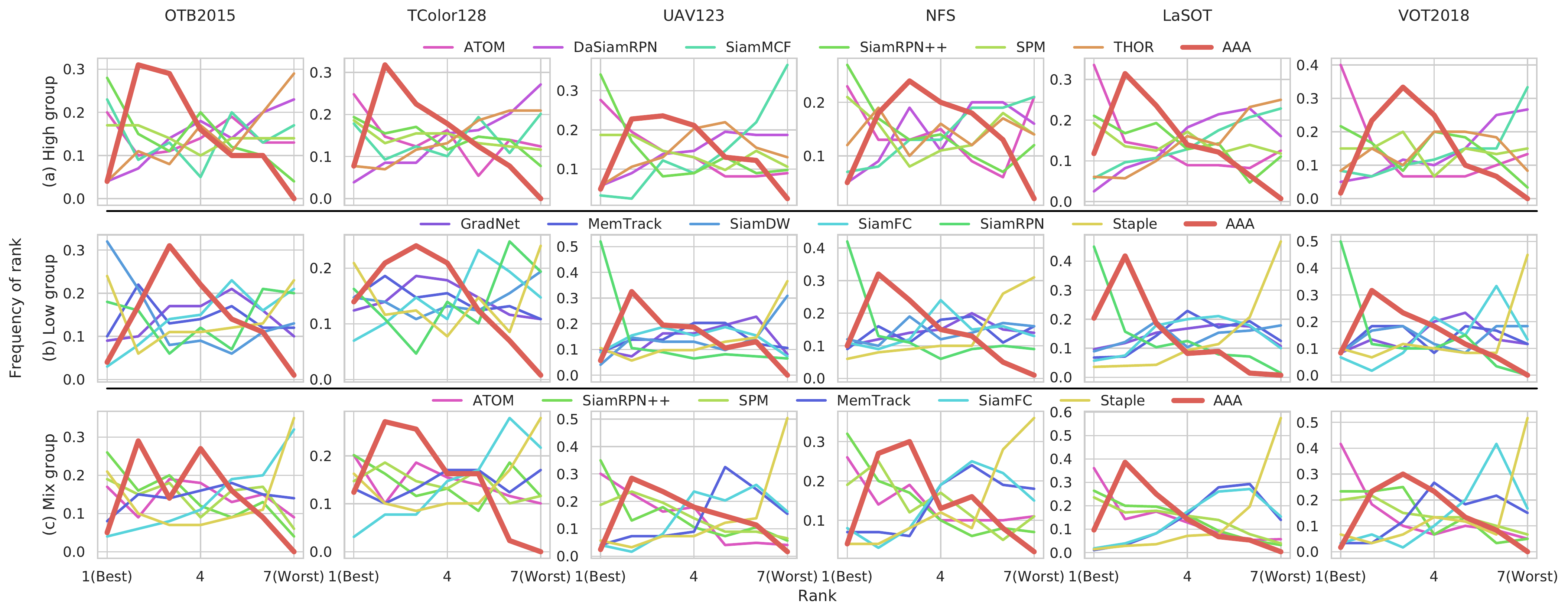}\vspace{-3mm}
\caption{Rank histograms of (a)~\textit{High}, (b)~\textit{Low}, and (c)~\textit{Mix} groups.}
\label{fig:rank}
\end{figure*}

\subsection{Experimental setup}
\subsubsection{Experts and comparative methods}
In relation to experts and comparative methods,
twelve state-of-the-art online trackers were employed (ATOM~\cite{Danelljan_2019}, DaSiamRPN~\cite{Zhu_2018}, GradNet~\cite{Li_2019G}, MemTrack~\cite{Yang_2018}, SiamDW~\cite{Zhang_2019}, SiamFC~\cite{bertinetto2016fully}, SiamMCF~\cite{Morimitsu2018MultipleCF}, SiamRPN~\cite{Li_2018}, SiamRPN++~\cite{Li_2019}, SPM~\cite{Wang_2019}, Staple~\cite{Bertinetto_2016}, and THOR~\cite{Sauer2019BMVC}).
For fair comparison and better performance, parameters optimized by the authors of the individual experts were applied.\par
In addition to these online trackers, AAA was also compared with the MCCT~\cite{Wang_2018} and HDT*~\cite{Qi2019HedgingDF}, aggregation-based tracking methods as detailed in Appendix~\ref{sec:append_mcct}
and Appendix~\ref{sec:append_hdt}, respectively.
Other naive aggregation-based methods referred to as ``Random'' and ``Max'' were also examined. ``Random'' randomly selects an expert estimation for each frame, while ``Max'' selects the estimation most similar to the template image for each frame. ``Max'' is the same as AAA when each frame is an anchor frame and thus feedback is given for each frame.

\subsubsection{Benchmark datasets}
The proposed AAA and the comparative methods were evaluated with OTB2015~\cite{Wu2015ObjectTB}, TColor128~\cite{Liang_2015}, UAV123~\cite{Mueller_2016}, NFS~\cite{Galoogahi_2017}, and LaSOT~\cite{Fan_2019}. OTB2015 is a popular benchmark dataset for evaluating online trackers, consisting of 100 image sequences including gray-scale image sequences. TColor128 contains 128 color image sequences, and is specifically designed for evaluation of color-enhanced trackers. VOT2018 is a dataset produced for competition, and consists of 60 image sequences. 
UAV123 consists of 123 image sequences taken from unmanned aerial vehicles. NFS consists of 100 image sequences captured with a higher frame rate of 240fps. LaSOT is the largest benchmark dataset among the above, and is divied into ``training'' and ``testing'' subsets. Here, image sequences from the testing subset containing 280 image sequences were used.

\subsubsection{Performance monitoring}
For performance monitoring, the area-under-the-curve (AUC) score and average distance precision (DP) were referenced as standard metrics~\cite{Wu_2013}. AUC (referred to as AO in VOT2018) is derived using ``success plot'' for the performance curve. This plot is based on evaluation of the ratio of frames where the IoU with the ground truth is larger than the threshold value ($\in [0,1]$). Fig.~\ref{fig:high_curve} shows examples of the success plot.\par
DP is derived from ``precision plot'' based on evaluation of the ratio of frames where the geometric distance between the location determined and the ground-truth location is less than the threshold value ($\in [0,50]$ pixels). Fig.~\ref{fig:high_curve} also shows examples of the precision plot. DP is eventually determined as the value of the precision plot at the threshold 20 based on \cite{Wu_2013}.\par
In addition to AUC and DP, the performance rank of individual trackers for each image sequence was used. Since there are $N$ experts and AAA, the rank varies from 1 (the best) to $N+1$ (the worst). If AAA successfully follows the best expert for arbitrary image sequences, it will be frequently ranked second. \par
Performance was evaluated under the ``strictly-online'' conditions described here. First, no tracker has any prior information on the total frame length $T$. Second, the no-reset evaluation protocol was used; in some experimental evaluations (e.g., VOT2018) a reset-based evaluation protocol is employed, where failed trackers (with zero IoU with ground-truth) can restart from the correct location a fixed number of frames later. In no-reset evaluation, however, the failed tracker continues tracking rather than being reset.

\begin{table*}[!t]
\centering
\begin{threeparttable}
\caption{Tracking Accuracy by \textit{Low} Group\vspace{-2mm}}
\label{table:low_score}
\begin{tabular}{c|c|c|c|c|c|c|c|c|c|c|c}
\hline
\multirow{2}{*}{Tracker} & \multicolumn{2}{c|}{OTB2015} & \multicolumn{2}{c|}{TColor128} & \multicolumn{2}{c|}{UAV123} & \multicolumn{2}{c|}{NFS} & \multicolumn{2}{c|}{LaSOT} & VOT2018 \\
  & AUC & DP & AUC & DP & AUC & DP & AUC & DP & AUC & DP & AUC \\
\hline\hline
GradNet~\cite{Li_2019G} & 0.63 & 0.85 & {\color{blue} \textit{0.56}} & {\color{blue} \textit{0.76}} & 0.51 & 0.74 & 0.51 & 0.64 & 0.36 & 0.38 & 0.40 \\
MemTrack~\cite{Yang_2018} & 0.63 & 0.82 & 0.54 & 0.74 & 0.49 & 0.70 & 0.50 & 0.61 & 0.34 & 0.35 & 0.39 \\
SiamDW~\cite{Zhang_2019} & {\color{red} \textbf{0.67}} & {\color{red} \textbf{0.91}} & 0.53 & 0.75 & 0.46 & 0.69 & 0.50 & 0.63 & 0.35 & 0.34 & 0.37 \\
SiamFC~\cite{bertinetto2016fully} & 0.59 & 0.78 & 0.52 & 0.70 & 0.51 & 0.74 & 0.51 & 0.60 & 0.35 & 0.36 & 0.33 \\
SiamRPN~\cite{Li_2018} & 0.63 & 0.83 & 0.52 & 0.71 & {\color{red} \textbf{0.58}} & {\color{blue} \textit{0.77}} & {\color{blue} \textit{0.56}} & {\color{blue} \textit{0.66}} & {\color{red} \textbf{0.45}} & {\color{blue} \textit{0.45}} & {\color{red} \textbf{0.48}} \\
Staple~\cite{Bertinetto_2016} & 0.60 & 0.79 & 0.51 & 0.68 & 0.45 & 0.64 & 0.41 & 0.48 & 0.24 & 0.23 & 0.30 \\
\hdashline
HDT*~\cite{Qi2019HedgingDF} & - & 0.75 & - & 0.63 & - & 0.64 & - & 0.48 & - & 0.26 & - \\
MCCT~\cite{Wang_2018} & 0.59 & 0.79 & 0.49 & 0.66 & 0.50 & 0.70 & 0.51 & 0.63 & 0.32 & 0.34 & 0.32 \\
Random & 0.62 & 0.83 & 0.53 & 0.72 & 0.50 & 0.71 & 0.50 & 0.60 & 0.35 & 0.35 & 0.38 \\
Max & 0.63 & 0.83 & 0.52 & 0.71 & 0.51 & 0.73 & 0.53 & 0.64 & 0.35 & 0.35 & 0.38 \\
AAA(Proposed) & {\color{blue} \textit{0.66}} & {\color{blue} \textit{0.87}} & {\color{red} \textbf{0.59}} & {\color{red} \textbf{0.82}} & {\color{blue} \textit{0.56}} & {\color{red} \textbf{0.78}} & {\color{red} \textbf{0.58}} & {\color{red} \textbf{0.69}} & {\color{blue} \textit{0.45}} & {\color{red} \textbf{0.46}} & {\color{blue} \textit{0.45}} \\
\hline
\end{tabular}
    \begin{tablenotes}
      \item{$\cdot$}  See the notes below Table~\ref{table:high_score} for the details. 
    \end{tablenotes}
\end{threeparttable}
\end{table*}

\begin{table*}[!t]
\centering
\begin{threeparttable}
\caption{Tracking Accuracy by \textit{Mix} Group\vspace{-2mm}}
\label{table:mix_score}
\begin{tabular}{c|c|c|c|c|c|c|c|c|c|c|c}
\hline
\multirow{2}{*}{Tracker} & \multicolumn{2}{c|}{OTB2015} & \multicolumn{2}{c|}{TColor128} & \multicolumn{2}{c|}{UAV123} & \multicolumn{2}{c|}{NFS} & \multicolumn{2}{c|}{LaSOT} & VOT2018 \\
  & AUC & DP & AUC & DP & AUC & DP & AUC & DP & AUC & DP & AUC \\
\hline\hline
ATOM~\cite{Danelljan_2019} & 0.67 & 0.87 & {\color{blue} \textit{0.60}} & {\color{blue} \textit{0.81}} & {\color{red} \textbf{0.62}} & {\color{red} \textbf{0.82}} & 0.58 & 0.69 & {\color{red} \textbf{0.51}} & {\color{blue} \textit{0.51}} & {\color{red} \textbf{0.52}} \\
SiamRPN++~\cite{Li_2019} & {\color{red} \textbf{0.69}} & {\color{red} \textbf{0.90}} & 0.58 & 0.77 & {\color{blue} \textit{0.60}} & 0.80 & {\color{red} \textbf{0.60}} & {\color{red} \textbf{0.74}} & 0.49 & 0.51 & {\color{blue} \textit{0.50}} \\
SPM~\cite{Wang_2019} & 0.67 & 0.87 & 0.58 & 0.79 & 0.59 & 0.77 & 0.57 & 0.67 & 0.47 & 0.48 & 0.48 \\
MemTrack~\cite{Yang_2018} & 0.63 & 0.82 & 0.54 & 0.74 & 0.49 & 0.70 & 0.50 & 0.61 & 0.34 & 0.35 & 0.39 \\
SiamFC~\cite{bertinetto2016fully} & 0.59 & 0.78 & 0.52 & 0.70 & 0.51 & 0.74 & 0.51 & 0.60 & 0.35 & 0.36 & 0.33 \\
Staple~\cite{Bertinetto_2016} & 0.60 & 0.79 & 0.51 & 0.68 & 0.45 & 0.64 & 0.41 & 0.48 & 0.24 & 0.23 & 0.30 \\
\hdashline
HDT*~\cite{Qi2019HedgingDF} & - & 0.78 & - & 0.68 & - & 0.67 & - & 0.50 & - & 0.28 & - \\
MCCT~\cite{Wang_2018} & 0.60 & 0.78 & 0.50 & 0.66 & 0.53 & 0.71 & 0.53 & 0.65 & 0.36 & 0.38 & 0.33 \\
Random & 0.64 & 0.84 & 0.55 & 0.75 & 0.54 & 0.74 & 0.53 & 0.63 & 0.40 & 0.41 & 0.42 \\
Max & 0.65 & 0.84 & 0.56 & 0.76 & 0.55 & 0.76 & 0.57 & 0.68 & 0.40 & 0.42 & 0.43 \\
AAA(Proposed) & {\color{blue} \textit{0.68}} & {\color{blue} \textit{0.89}} & {\color{red} \textbf{0.62}} & {\color{red} \textbf{0.83}} & 0.60 & {\color{blue} \textit{0.81}} & {\color{blue} \textit{0.59}} & {\color{blue} \textit{0.71}} & {\color{blue} \textit{0.51}} & {\color{red} \textbf{0.53}} & 0.49 \\
\hline
\end{tabular}
    \begin{tablenotes}
      \item{$\cdot$}   See the notes below Table~\ref{table:high_score} for the details.
    \end{tablenotes}
\end{threeparttable}
\end{table*}

\subsubsection{Hyper-parameter search}\label{sec:threshold}
The threshold $\theta$, which is just one hyper-parameter in the proposed method, is optimized using the GOT10K~\cite{Huang_2019}, Generic Object Tracking Benchmark.
Specifically, the AUC score of AAA with the expert group is first evaluated by changing $\theta$ from 0.6 to 0.9 at 0.01 intervals with GOT10K. Then, the $\theta$ for the highest AUC score is chosen for evaluation of the other datasets. Appendix~\ref{sec:append_threshold} details the procedure.
It should be noted that GOT10K was not used in the performance evaluation experiments described below or in training of individual experts.

\subsection{Quantitative evaluation using three expert groups with different performance}
Comprehensive experiments were conducted to determine whether the proposed method can be applied to properly aggregate various experts and achieve near-best performance (i.e., the performance similar to that of the best expert) without attentive expert selection. This section outlines quantitative evaluation conducted using three expert groups (referred to as \textit{High}, \textit{Low}, and \textit{Mix}) with different performance characteristics as follows:
\begin{itemize}
    \item \textit{High} group consists of six higher-performance experts (ATOM, DaSiamRPN, SiamMCF, SiamRPN++, SPM, and THOR). 
    \item \textit{Low} group consists of six lower-performance experts (GradNet, MemTrack, SiamDW, SiamFC, SiamRPN, and Staple), and was examined to determine whether the proposed method can be applied to follow the best experts among lower-performance experts.
    \item \textit{Mix} group consists of the three higher-performance experts (ATOM, SiamRPN++, and SPM) and the three lower-performance experts (MemTrack, SiamFC, and Staple). Observation of the proposed method with this is importance in verifying that the method supports automatic selection of higher-performance experts while helping to eliminate erroneous estimations from lower-performance experts. 
\end{itemize}
\par
Table~\ref{table:high_score} shows the average performance of the six experts in the \textit{High} group and their aggregations. As noted, AUC and DP are derived
from the success plot and the precision plot as shown in Fig.~\ref{fig:high_curve}. 
AAA outperformed state-of-the-art trackers in most datasets and demonstrated the best performance for all datasets except for UAV123, and even then was only marginally second best. 
\par
Fig.~\ref{fig:rank}~(a) shows an image sequence-level rank histogram for six trackers in the \textit{High} group and AAA. The histogram is normalized based on the number of image sequences in each dataset. As seen in Fig.~\ref{fig:hard_tracking}, no tracker consistently achieved the best performance, and the best expert drastically changed over image sequences even within the same dataset. As also shown in Fig.~\ref{fig:rank}~(a), even a very good tracker (such as SiamRPN++) is sometimes worst-ranked (i.e., 7th). In contrast, AAA was the second- or third-best tracker for most image sequences, although it was not often the best. More importantly, it was rarely ranked lowly. These experimental highlight the importance of a regret bound guaranteeing a minimal difference between performance of the proposed method and the best expert.\par

Table~\ref{table:high_score} also shows that the performance of the other aggregation-based trackers (HDT*, MCCT, Random, and Max) was lower than that of AAA despite their aggregation of the same experts. As detailed in Appendix~\ref{sec:append_hdt} and Appendix~\ref{sec:append_mcct}, HDT* and MCCT did not show ideal performance in the study's stringent experimental setup with aggregation of arbitrary experts. Max, which relies on the feedback given for each frame, sometimes demonstrated the worst performance. 

\begin{table*}[!t]
\centering
\begin{threeparttable}
\caption{Tracking Accuracy by \textit{SiamDW} Group\vspace{-2mm}}
\label{table:siamdw_score}
\begin{tabular}{c|c|c|c|c|c|c|c|c|c|c|c}
\hline
\multirow{2}{*}{Tracker} & \multicolumn{2}{c|}{OTB2015} & \multicolumn{2}{c|}{TColor128} & \multicolumn{2}{c|}{UAV123} & \multicolumn{2}{c|}{NFS} & \multicolumn{2}{c|}{LaSOT} & VOT2018 \\
  & AUC & DP & AUC & DP & AUC & DP & AUC & DP & AUC & DP & AUC \\
\hline\hline
SiamDW/SiamFCRes22/OTB & 0.64 & 0.84 & {\color{blue} \textit{0.58}} & {\color{blue} \textit{0.79}} & 0.51 & 0.73 & 0.52 & 0.64 & {\color{blue} \textit{0.38}} & {\color{blue} \textit{0.39}} & 0.38 \\
SiamDW/SiamFCIncep22/OTB & 0.61 & 0.81 & 0.55 & 0.76 & 0.50 & 0.72 & 0.51 & 0.64 & 0.36 & 0.38 & 0.35 \\
SiamDW/SiamFCNext22/OTB & 0.62 & 0.82 & 0.57 & 0.76 & 0.49 & 0.71 & 0.51 & 0.63 & 0.37 & 0.38 & 0.32 \\
SiamDW/SiamRPNRes22/OTB & {\color{red} \textbf{0.67}} & {\color{red} \textbf{0.91}} & 0.53 & 0.75 & 0.46 & 0.69 & 0.50 & 0.63 & 0.35 & 0.34 & 0.37 \\
SiamDW/SiamFCRes22/VOT & 0.63 & 0.84 & 0.56 & 0.77 & 0.51 & 0.73 & 0.52 & 0.65 & 0.36 & 0.37 & 0.37 \\
SiamDW/SiamFCIncep22/VOT & 0.60 & 0.80 & 0.54 & 0.75 & 0.50 & 0.73 & 0.49 & 0.61 & 0.35 & 0.36 & 0.35 \\
SiamDW/SiamFCNext22/VOT & 0.61 & 0.81 & 0.54 & 0.74 & 0.49 & 0.72 & 0.51 & 0.63 & 0.35 & 0.38 & 0.34 \\
SiamDW/SiamRPNRes22/VOT & 0.66 & {\color{blue} \textit{0.90}} & 0.53 & 0.74 & 0.46 & 0.69 & 0.51 & {\color{blue} \textit{0.66}} & 0.35 & 0.35 & {\color{red} \textbf{0.43}} \\
\hdashline
HDT*~\cite{Qi2019HedgingDF} & - & 0.81 & - & 0.70 & - & 0.67 & - & 0.57 & - & 0.30 & - \\
MCCT~\cite{Wang_2018} & 0.63 & 0.83 & 0.54 & 0.74 & 0.50 & 0.71 & 0.51 & 0.65 & 0.34 & 0.36 & 0.34 \\
Random & 0.63 & 0.84 & 0.55 & 0.76 & 0.49 & 0.71 & 0.51 & 0.64 & 0.36 & 0.37 & 0.37 \\
Max & 0.64 & 0.86 & 0.55 & 0.76 & {\color{blue} \textit{0.51}} & {\color{blue} \textit{0.74}} & {\color{blue} \textit{0.53}} & 0.66 & 0.36 & 0.36 & 0.37 \\
AAA(Proposed) & {\color{blue} \textit{0.66}} & 0.88 & {\color{red} \textbf{0.60}} & {\color{red} \textbf{0.82}} & {\color{red} \textbf{0.52}} & {\color{red} \textbf{0.75}} & {\color{red} \textbf{0.55}} & {\color{red} \textbf{0.68}} & {\color{red} \textbf{0.42}} & {\color{red} \textbf{0.43}} & {\color{blue} \textit{0.42}} \\
\hline
\end{tabular}
\begin{tablenotes}
      \item{$\cdot$}  Depending on the backbone network used and what dataset it is designed for, the names of experts are written as follows: ``base algorithm''/``backbone network''/``target benchmark''. The authors of SiamDW and SiamRPN++ propose different hyper-parameters according to the target benchmark even for the same backbone network.
      \item{$\cdot$}  See the notes below Table~\ref{table:high_score} for the other details. 
    \end{tablenotes}
\end{threeparttable}
\end{table*}

\begin{table*}[!t]
\centering
\begin{threeparttable}
\caption{Tracking Accuracy by \textit{SiamRPN++} Group\vspace{-2mm}}
\label{table:siamrpn_score}
\begin{tabular}{c|c|c|c|c|c|c|c|c|c|c|c}
\hline
\multirow{2}{*}{Tracker} & \multicolumn{2}{c|}{OTB2015} & \multicolumn{2}{c|}{TColor128} & \multicolumn{2}{c|}{UAV123} & \multicolumn{2}{c|}{NFS} & \multicolumn{2}{c|}{LaSOT} & VOT2018 \\
  & AUC & DP & AUC & DP & AUC & DP & AUC & DP & AUC & DP & AUC \\
\hline\hline
SiamRPN++/AlexNet/VOT & 0.66 & 0.87 & 0.57 & 0.77 & 0.58 & 0.77 & 0.54 & 0.65 & 0.45 & 0.45 & 0.47 \\
SiamRPN++/AlexNet/OTB & 0.66 & 0.86 & 0.55 & 0.75 & 0.58 & 0.78 & 0.54 & 0.66 & 0.43 & 0.43 & 0.45 \\
SiamRPN++/ResNet-50/VOT & 0.65 & 0.86 & 0.56 & 0.75 & 0.61 & 0.80 & 0.58 & 0.71 & 0.50 & 0.51 & 0.51 \\
SiamRPN++/ResNet-50/OTB & {\color{red} \textbf{0.69}} & {\color{red} \textbf{0.90}} & {\color{blue} \textit{0.58}} & 0.77 & 0.60 & 0.80 & 0.60 & {\color{blue} \textit{0.74}} & 0.49 & 0.51 & 0.50 \\
SiamRPN++/ResNet-50/VOTLT & 0.63 & 0.84 & 0.58 & {\color{blue} \textit{0.79}} & 0.61 & {\color{blue} \textit{0.81}} & 0.56 & 0.68 & {\color{blue} \textit{0.52}} & {\color{blue} \textit{0.54}} & {\color{blue} \textit{0.51}} \\
SiamRPN++/MobileNetV2/VOT & 0.65 & 0.86 & 0.56 & 0.76 & 0.60 & 0.79 & 0.57 & 0.70 & 0.45 & 0.46 & 0.50 \\
SiamRPN++/SiamMask/VOT & 0.65 & 0.85 & 0.54 & 0.73 & 0.60 & 0.80 & 0.58 & 0.72 & 0.47 & 0.48 & 0.48 \\
\hdashline
HDT*~\cite{Qi2019HedgingDF} & - & 0.80 & - & 0.68 & - & 0.71 & - & 0.59 & - & 0.38 & - \\
MCCT~\cite{Wang_2018} & 0.64 & 0.84 & 0.55 & 0.75 & {\color{blue} \textit{0.61}} & 0.81 & 0.59 & 0.72 & 0.48 & 0.50 & 0.45 \\
Random & 0.66 & 0.86 & 0.56 & 0.76 & 0.60 & 0.79 & 0.57 & 0.69 & 0.47 & 0.48 & 0.49 \\
Max & 0.66 & 0.87 & 0.56 & 0.76 & 0.60 & 0.80 & {\color{blue} \textit{0.60}} & 0.73 & 0.47 & 0.48 & 0.50 \\
AAA(Proposed) & {\color{blue} \textit{0.68}} & {\color{blue} \textit{0.89}} & {\color{red} \textbf{0.61}} & {\color{red} \textbf{0.83}} & {\color{red} \textbf{0.64}} & {\color{red} \textbf{0.85}} & {\color{red} \textbf{0.61}} & {\color{red} \textbf{0.74}} & {\color{red} \textbf{0.54}} & {\color{red} \textbf{0.56}} & {\color{red} \textbf{0.52}} \\
\hline
\end{tabular}
\begin{tablenotes}
      \item{$\cdot$}  ``VOTLT'' is a dataset consisting of relatively long-term image sequences rather than VOT~\cite{Kristan_2016}.
     \item{$\cdot$} See the notes below Table~\ref{table:high_score} for the other details. 
    \end{tablenotes}
\end{threeparttable}
\end{table*}

Table~\ref{table:low_score} and Fig.~\ref{fig:rank}~(b) show quantitative evaluation of the trackers in the \textit{Low} group and their aggregations. AAA again demonstrated at least the second-best average performance for all datasets even with aggregation of experts in the \textit{Low} group. This means that AAA automatically finds and follows the best among lower-performance experts as expected. It should be emphasized that the other aggregation strategies suffered from lower expert performance.\par

As shown in Table~\ref{table:mix_score} and Fig.~\ref{fig:rank}~(c), AAA was at least the third-best of all trackers with experts in the \textit{Mix} group, and there were significant performance gaps between the three higher- and lower-performance experts in this group. Here, a selection of lower-performance experts drastically degrades aggregation performance. AAA, however, automatically and successfully selected higher-performance experts and was the second- or third-best tracker for most image sequences. This indicates that AAA is not disturbed even when several experts do not perform well, making it a very practical aggregation-based tracker.

\begin{figure*}[!t]
\captionsetup{farskip=0pt}
\includegraphics[width=1\linewidth]{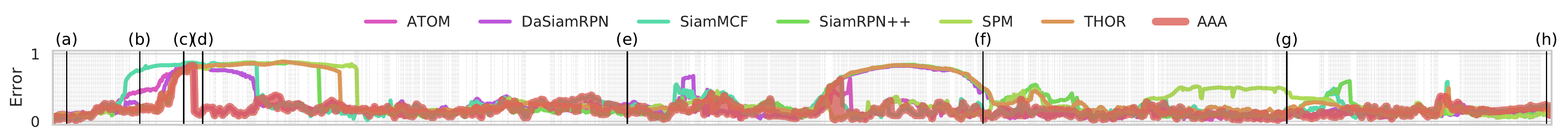}
\includegraphics[width=1\linewidth]{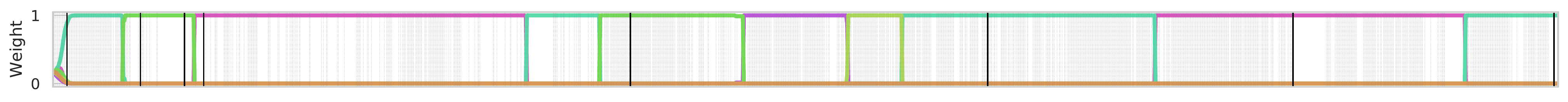}
\subfloat[Frame 14]{\label{fig:example-a}\includegraphics[width=0.125\linewidth]{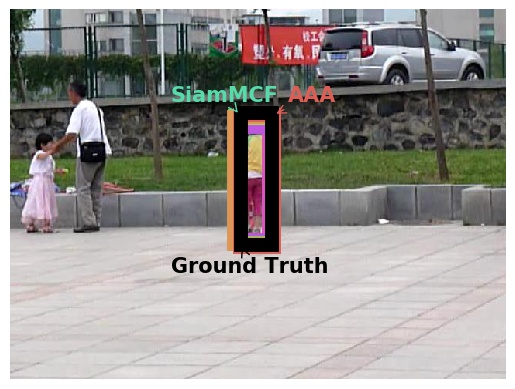}}
\subfloat[Frame 87]{\label{fig:example-b}\includegraphics[width=0.125\linewidth]{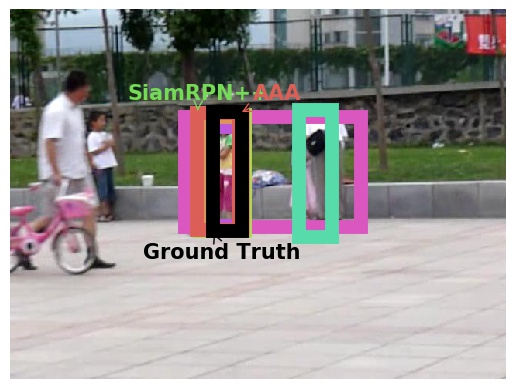}}
\subfloat[Frame 131]{\label{fig:example-c}\includegraphics[width=0.125\linewidth]{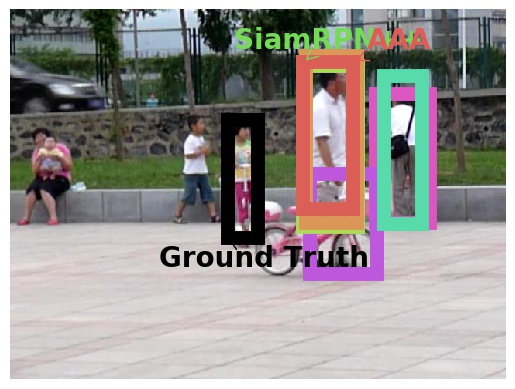}}
\subfloat[Frame 150]{\label{fig:example-d}\includegraphics[width=0.125\linewidth]{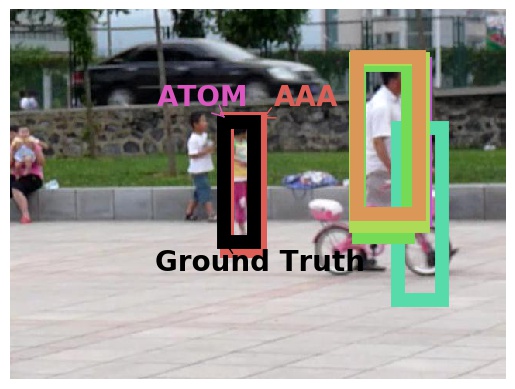}}
\subfloat[Frame 575]{\label{fig:example-e}\includegraphics[width=0.125\linewidth]{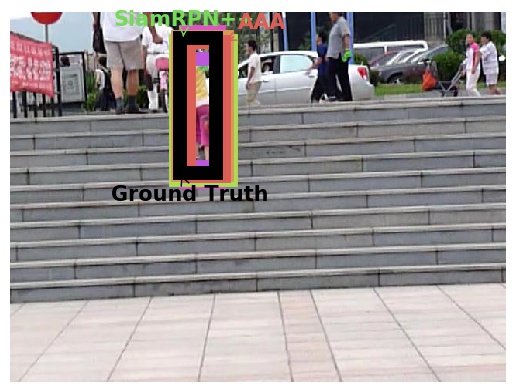}}
\subfloat[Frame 931]{\label{fig:example-f}\includegraphics[width=0.125\linewidth]{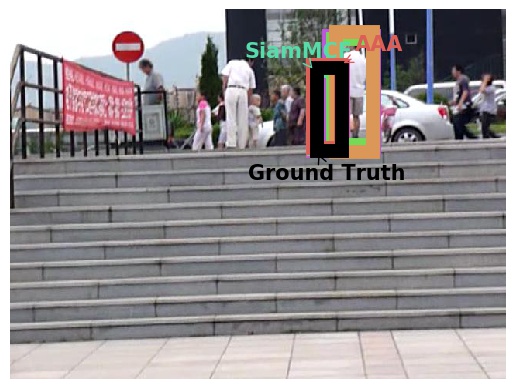}}
\subfloat[Frame 1235]{\label{fig:example-g}\includegraphics[width=0.125\linewidth]{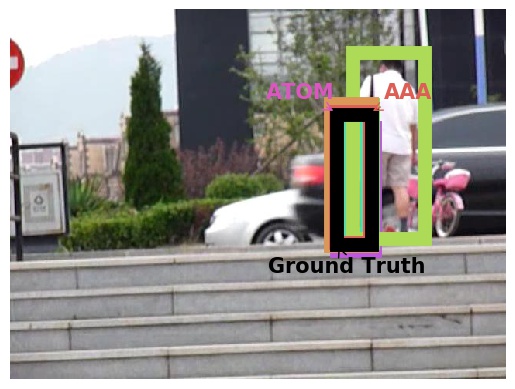}}
\subfloat[Frame 1495]{\label{fig:example-h}\includegraphics[width=0.125\linewidth]{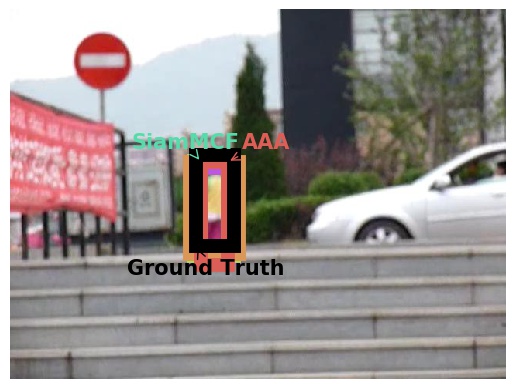}}
\vspace{5pt}
\hrule
\vspace{5pt}
\includegraphics[width=1\linewidth]{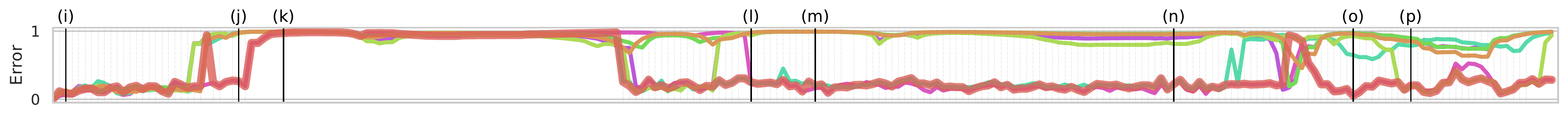}
\includegraphics[width=1\linewidth]{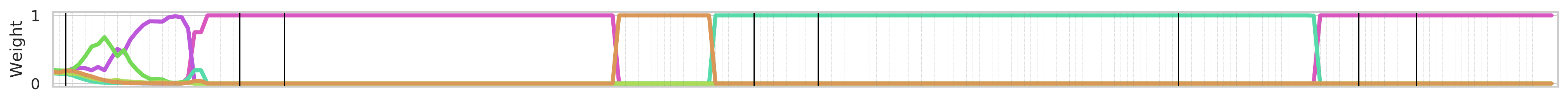}
\subfloat[Frame 2]{\label{fig:example-i}\includegraphics[width=0.125\linewidth]{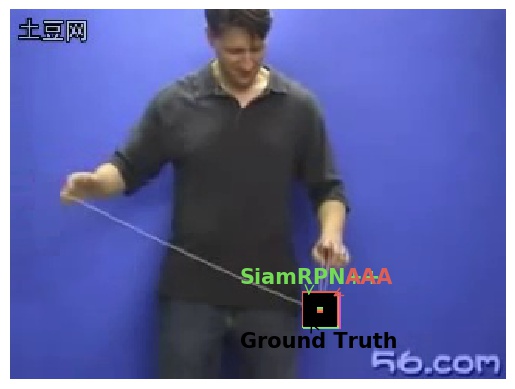}}
\subfloat[Frame 29]{\label{fig:example-j}\includegraphics[width=0.125\linewidth]{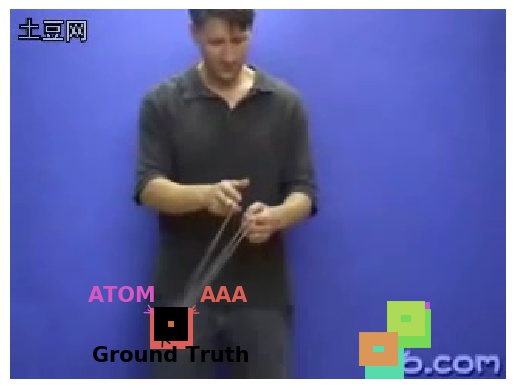}}
\subfloat[Frame 36]{\label{fig:example-k}\includegraphics[width=0.125\linewidth]{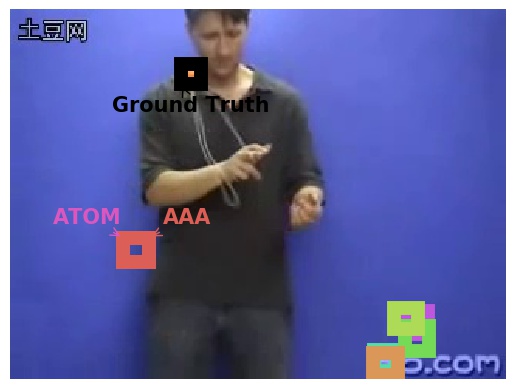}}
\subfloat[Frame 109]{\label{fig:example-l}\includegraphics[width=0.125\linewidth]{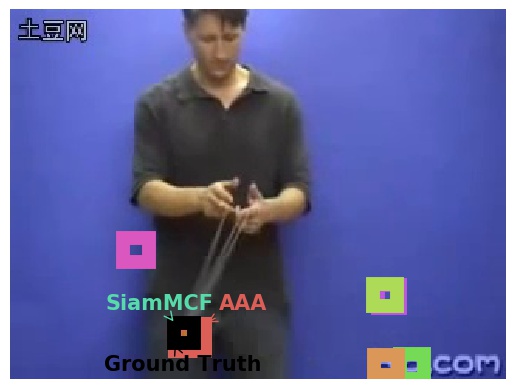}}
\subfloat[Frame 119]{\label{fig:example-m}\includegraphics[width=0.125\linewidth]{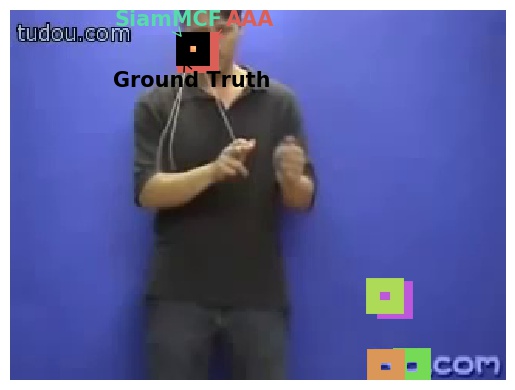}}
\subfloat[Frame 175]{\label{fig:example-n}\includegraphics[width=0.125\linewidth]{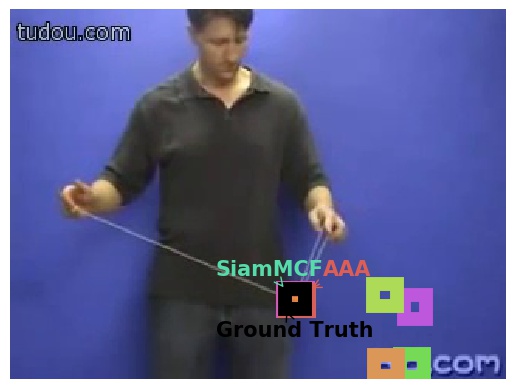}}
\subfloat[Frame 203]{\label{fig:example-o}\includegraphics[width=0.125\linewidth]{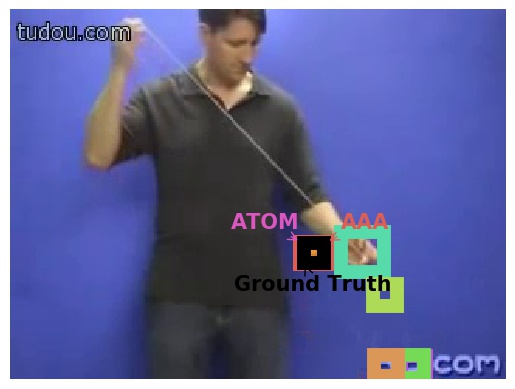}}
\subfloat[Frame 212]{\label{fig:example-p}\includegraphics[width=0.125\linewidth]{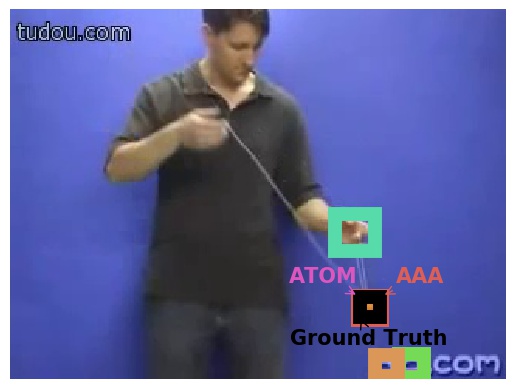}}
\caption{{\bf Tracking examples.} The top is ``Girl2'' in OTB2015 and the bottom is ``Yo-yos\_ce1'' in TColor128. 
For the picked-up frame,
we annotate the true target location (black), the proposed method (red), and the expert which has the highest weight.
In all the graphs, gray vertical lines indicate the location of the anchor frames.}
\label{fig:example}
\end{figure*}

\subsection{Quantitative evaluation with experts generated from a single tracking method}
As explained in \ref{sec:effect_of_N}, Theorem \ref{theo:regret_cdelay} guarantees that the use of more experts does not result in any significant degradation of AAA performance. One strategy here is to employ a wide variety of tracking methods based on different algorithms as per the experiments detailed in the previous section. However, due to the difficulty of arranging such a variety, it is preferable to employ a single tracking method and generate various versions thereof by changing its internal parameters, like ~\cite{Zhang_2019, Li_2019}.\par
For evaluation of AAA with the second strategy, different versions of the lower-performance SiamDW (collectively referred to as the \textit{SiamDW} group) were generated by changing the related parameter sets (e.g., backbone networks, the weight of the network, and hyper-parameters) based on the suggestions of the original paper~\cite{Zhang_2019}. The higher-performance tracker SiamRPN++ was also adapted to generate the \textit{SiamRPN++} group. \par
Tables~\ref{table:siamdw_score} and \ref{table:siamrpn_score} show the results from the \textit{SiamDW} and \textit{SiamRPN++} groups.
Even with experts generated from a single tracking method, AAA outperformed the individual experts in these groups, and the negative effect of increasing the number of experts $N$ on the regret bound was therefore not significant. These results are promising for practical application. Rather than focusing on the internal parameters of individual experts, it is simply necessary to have more experts generated with different internal parameters. 

\subsection{Tracking examples}
In Fig.~\ref{fig:example}, two tracking examples from the \textit{High} group are shown to demonstrate how AAA tries to track the best expert and achieve similar performance by updating the weights $w_i^t$ adaptively. Specifically, it shows change in the overlap errors (IOU) and the weights of experts and AAA for ``Girl2'' in OTB2015 and ``Yo-yos\_ce1'' in TColor128.\par
For ``Girl2'', all experts successfully tracked the target object at (a). However, because of the occlusion between (b) and (c), the experts failed to track at (c) and only ATOM properly tracked the target object at (d).
AAA also tracked the object well at (d), with an anchor frame being determined after the target object reappeared and high weight given to ATOM via appropriate feedback. At (e), (f) and (g), other experts were also able to properly track the target object, and AAA still achieved high performance. \par
For ``Yo-yos\_ce1'', tracker accuracy frequently varied throughout the sequence. For example, SiamMCF was accurate and ATOM was inaccurate at (l), but this situation was reversed at (o). Even here, AAA tried to follow the better tracker by changing the weights of the experts at the anchor frames, resulting in lower errors for most parts (except the period around (k), when all experts failed). AAA outperformed all the experts in this image sequence.

\section{When does AAA perform best on a dataset? -- A theoretical inspection\label{sec:average-performance}}
It is considered useful to know the conditions in which AAA outperforms experts as seen in the above experiments. However, Theorem \ref{theo:regret_cdelay} 
simply implies that the performance of AAA is similar to that of the best expert for an image sequence based on its regret bound. It does not indicate the conditions in which AAA outperforms experts for a certain sequence.\par
However, it is still possible to determine the conditions in which AAA outperforms experts {\em on average over an image sequence set $\mathcal{V}$}. The four notations are used to indicate these conditions. First, $\overline{R}^\mathcal{V}$ denotes the average regret of AAA over $\mathcal{V}$.
Second, {\em the overall-best expert $i^*$} is the expert whose total loss (or, equivalently, average loss) over $\mathcal{V}$ is the minimum among the $N$ experts, that is:
\begin{equation}
    i^* = \argmin_{i\in[1,N]} \sum_{v\in \mathcal{V}}\sum_{t=1}^T\ell( f_{i}^{v,t}, y^{v,t}), 
    \label{eq:overall-best}
\end{equation}
where a new suffix $v$ is attached to $f_i^t$ and $y_i^t$. Third,  $\mathcal{S}\subset\mathcal{V}$ is the set of image sequences where the overall-best expert is the best expert. Finally, a value $\delta$ is defined as:
\begin{equation}
\delta = \min_{v\in\mathcal{V}\setminus\mathcal{S}}\left[\sum_{t=1}^T\ell( f_{i^*}^{v,t}, y^{v,t})-\min_i\sum_{t=1}^T\ell(f_i^{v,t}, y^{v,t})\right].
\label{eq:delta}
\end{equation}
In this definition, for the sequence $v\in\mathcal{V}\setminus\mathcal{S}$ (i.e., the sequence $v$ for which the overall-best expert is not the best expert), the overall-best expert performs worse than the best expert with a loss value of $\delta$ or more.\par
From the proof given in Appendix~\ref{sec:append_prop}, the following proposition holds:
\begin{prop}\label{prop:ours}
AAA outperforms all the experts on average over the image sequence set ${\mathcal{V}}$ if the following condition is satisfied:
\begin{equation}
    \overline{R}^\mathcal{V} \leq \frac{|\mathcal{V}\setminus\mathcal{S}|}{|\mathcal{V}|}\delta.
    \label{eq:prop_cond}
\end{equation}
\end{prop}
\par
This proposition states that AAA performs better on average (i.e., $\overline{R}^\mathcal{V}$ is smaller) when it satisfies the condition (\ref{eq:prop_cond}) with a smaller $\mathcal{S}$ and/or a larger $\delta$.
The set $\mathcal{S}$ is smaller if the overall-best expert is the best expert for only a smaller number of sequences. The difference $\delta$ is larger if the performance of the overall-best expert degrades drastically at 
$v\in\mathcal{V}\setminus\mathcal{S}$.\par
From these discussions it can be concluded that {\em when there is no almighty tracker} (i.e., where $\mathcal{S}$ is small), AAA performs better than all experts over $\mathcal{V}$ based on appropriate aggregation. A large $\delta$ also indicates that AAA is {\em better with employment of various experts}, each of which can be an outstanding expert for certain sequences in $\mathcal{V}$

\section{Conclusion}
This paper proposes a tracking method referred to as the Adaptive Aggregation of Arbitrary trackers (AAA) for robust online tracking. The performance of individual trackers varies significantly with different image sequences, creating variations in simple aggregation strategies. The proposed AAA is based on adaptive expert aggregation (AEA), which demonstrates strong theoretical support in terms of ``regret'', with a theoretically bounded performance difference between AAA and the best tracker (referred to as the best expert). It should be emphasized that the best tracker for an image sequence is identified at the end of the sequence. This means that it is unknown which tracker will be the best when AAA aggregates the trackers for each frame; nevertheless, this theoretical support guarantees that the performance of AAA will be close to that of the best tracker.\par
An exhaustive experimental study on the large variations of benchmark datasets and trackers to be aggregated demonstrated that the proposed method provides state-of-the-art performance. As a theoretical guarantee, AAA performed similar to or better than the best tracker for each image sequence, and often outperformed trackers on average over a benchmark dataset. We also derived a condition that AAA becomes the best tracker on average.\par
Future work will focus on extension to multiple object tracking. It might seem straightforward to aggregate multiple object trackers, but is in fact challenging because it is not obvious how we determine anchor frames, define delayed feedback and design a new loss function. Other potentially worthwhile work involves the development of a methodology for organizing expert sets that give better performance with the proposed aggregation for a certain dataset. 

\appendices

\begin{figure}[!t]
\centering
\begin{tabular}{ccc}
\includegraphics[width=0.45\linewidth]{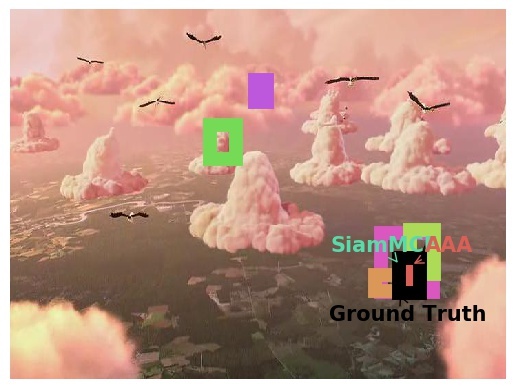} &
\includegraphics[width=0.45\linewidth]{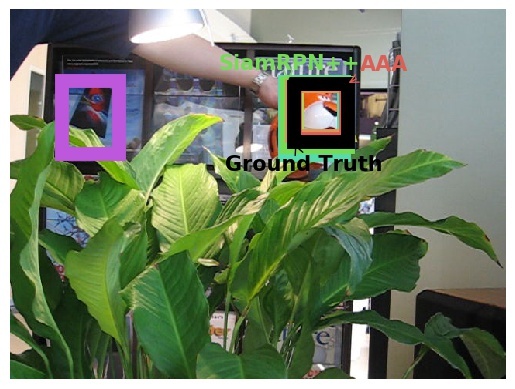} \\
\hline
\includegraphics[width=0.45\linewidth]{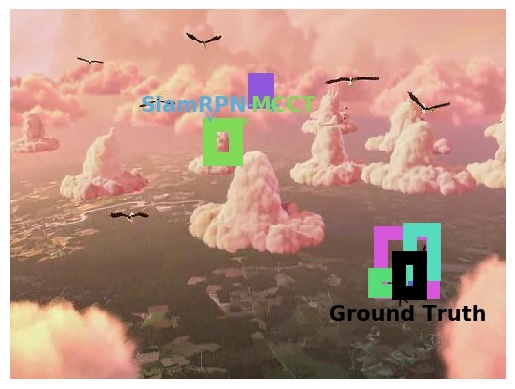} &
\includegraphics[width=0.45\linewidth]{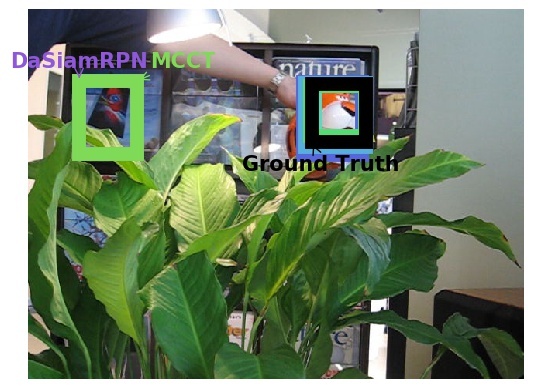}
\end{tabular}
\caption{{\bf Tracking example of (top) the proposed method and (bottom) MCCT.} We annotate the true target location (black), the proposed method and MCCT (red). The expert which has the highest weight is also annotated.} 
\label{fig:mcct_compare}
\end{figure}

\begin{figure}[!t]
\centering
\begin{tabular}{ccc}
\includegraphics[width=0.45\linewidth]{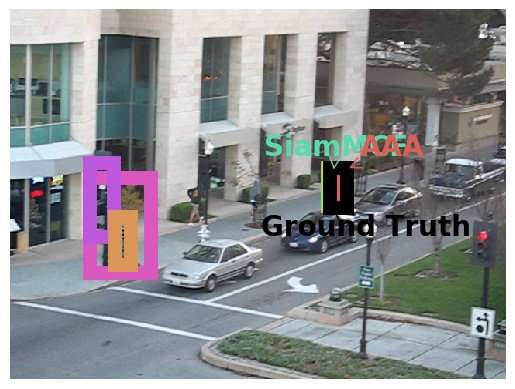} &
\includegraphics[width=0.45\linewidth]{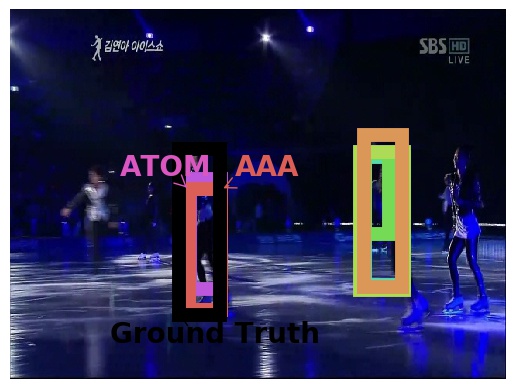} \\
\hline
\includegraphics[width=0.45\linewidth]{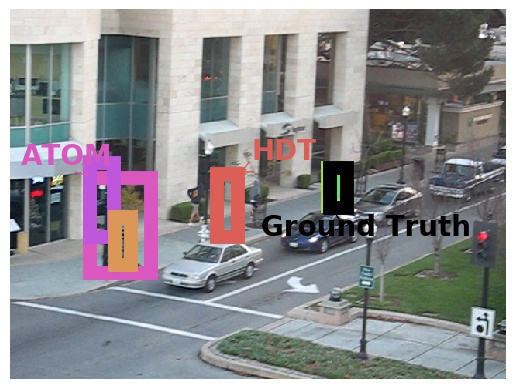} &
\includegraphics[width=0.45\linewidth]{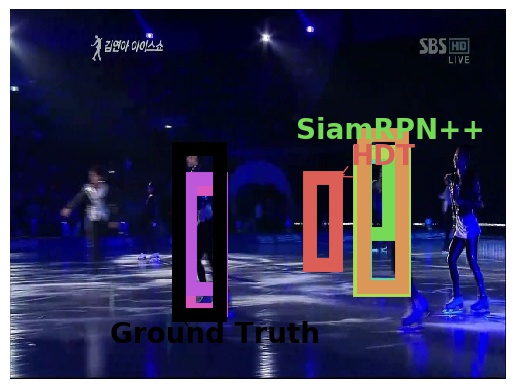}
\end{tabular}
\caption{{\bf Tracking example of (top) the proposed method and (bottom) HDT.} We annotate the true target location (black), the proposed method and HDT (red). The expert which has the highest weight is also annotated.} 
\label{fig:hdt_compare}
\end{figure}

\begin{figure*}[!t]
\centering
\includegraphics[width=1\linewidth]{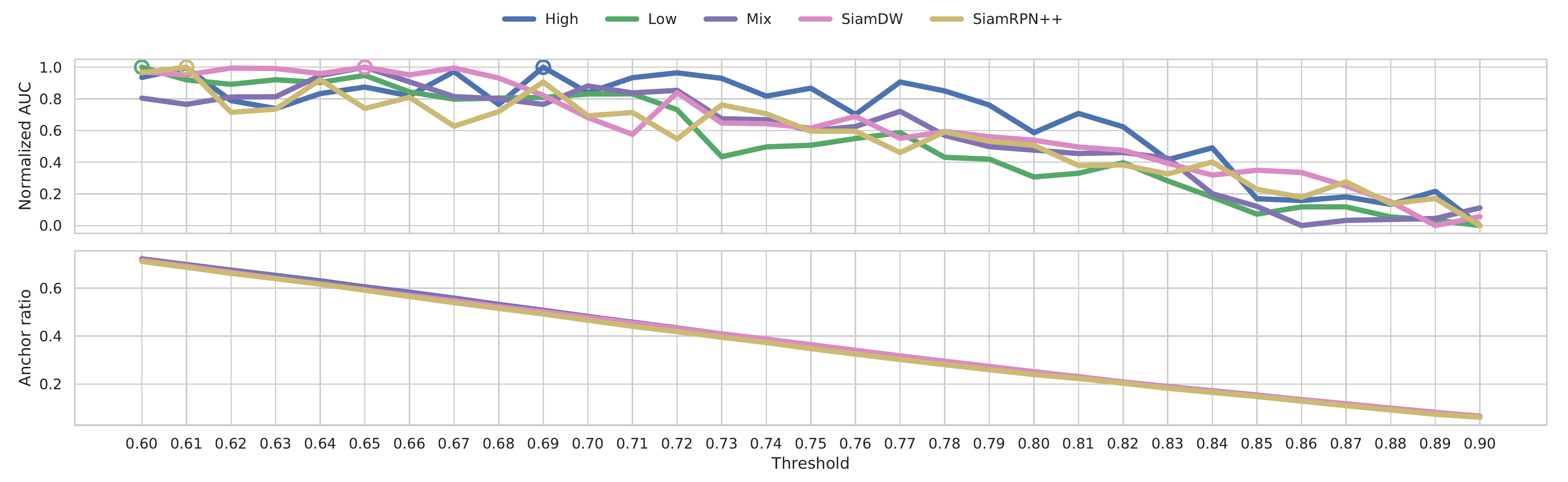}
\caption{\textbf{(top) Normalized AUC score and (bottom) anchor ratio for thresholds $\theta$.} For each comparison group, the AUC scores are normalized by min-max normalization. The threshold with the highest AUC score (i.e., 1) is indicated by a circle.}
\label{fig:theta}
\end{figure*}

\begin{table*}[!t]
\centering
\begin{threeparttable}
\caption{Tracking Accuracy at Anchor Frames by  \textit{High} Group\vspace{-2mm}}
\label{table:high_anchor}
\begin{tabular}{c|c|c|c|c|c|c|c|c|c|c|c}
\hline
\multirow{2}{*}{Tracker} & \multicolumn{2}{c|}{OTB2015} & \multicolumn{2}{c|}{TColor128} & \multicolumn{2}{c|}{UAV123} & \multicolumn{2}{c|}{NFS} & \multicolumn{2}{c|}{LaSOT} & VOT2018 \\
  & AUC & DP & AUC & DP & AUC & DP & AUC & DP & AUC & DP & AUC \\
\hline\hline
ATOM & 0.72 & 0.93 & {\color{blue} \textit{0.66}} & 0.88 & {\color{blue} \textit{0.71}} & {\color{blue} \textit{0.91}} & 0.66 & 0.80 & {\color{blue} \textit{0.66}} & {\color{blue} \textit{0.70}} & {\color{blue} \textit{0.61}} \\
DaSiamRPN & 0.70 & 0.92 & 0.61 & 0.84 & 0.68 & 0.88 & 0.64 & 0.79 & 0.59 & 0.62 & 0.57 \\
SiamMCF & 0.71 & 0.91 & 0.66 & 0.88 & 0.66 & 0.88 & 0.67 & 0.83 & 0.60 & 0.64 & 0.57 \\
SiamRPN++ & {\color{blue} \textit{0.73}} & {\color{blue} \textit{0.94}} & 0.64 & 0.85 & 0.70 & 0.89 & {\color{blue} \textit{0.68}} & {\color{blue} \textit{0.84}} & 0.64 & 0.69 & 0.60 \\
SPM & 0.72 & 0.92 & 0.65 & {\color{blue} \textit{0.89}} & 0.69 & 0.87 & 0.65 & 0.78 & 0.63 & 0.67 & 0.60 \\
THOR & 0.68 & 0.90 & 0.59 & 0.79 & 0.67 & 0.87 & 0.64 & 0.79 & 0.55 & 0.58 & 0.59 \\
\hdashline
AAA(Proposed) & {\color{red} \textbf{0.74}} & {\color{red} \textbf{0.94}} & {\color{red} \textbf{0.70}} & {\color{red} \textbf{0.92}} & {\color{red} \textbf{0.72}} & {\color{red} \textbf{0.93}} & {\color{red} \textbf{0.70}} & {\color{red} \textbf{0.87}} & {\color{red} \textbf{0.70}} & {\color{red} \textbf{0.76}} & {\color{red} \textbf{0.64}} \\
\hline
\end{tabular}
\end{threeparttable}
\end{table*}

\section{MCCT details\texorpdfstring{{}~\cite{Wang_2018}}{}\label{sec:append_mcct}}
An MCCT (Multi-Cue Correlation Tracker)~\cite{Wang_2018} evaluates and weights experts for every frame, and simply selects the one with the highest weight to determine the target location. For each expert, the MCCT evaluates the overlap ratio of bounding boxes from other experts and the variance of the overlap ratio in a short period. The expert with smaller variance is evaluated as more stable and therefore more reliable.
\par
However, with the employment of arbitrary experts, the MCCT may select a failed expert due to the nature of the above evaluation. More specifically, some experts that have already lost the target location will constantly have a zero overlap ratio and therefore zero variance; see Fig.~\ref{fig:mcct_compare} for examples. It can be seen that MCCT selects failed experts that maintain a zero overlap ratio. The authors' experiments revealed numerous such cases, and MCCT therefore demonstrated lower performance than other methods. 
Accordingly, to fully utilize the potential of MCCT, it is necessary to carefully choose experts with bounding boxes that tend to be close together. In fact, the study reported in the original paper~\cite{Wang_2018} used specific experts with adaptive updating to allow ROI sharing.

\section{HDT* details\texorpdfstring{~\cite{Qi2019HedgingDF}}{}\label{sec:append_hdt}}
This section outlines the implementation of HDT*~\cite{Qi2019HedgingDF} and related issues encountered in the study's experimental setup (with tough adversarial environments). As the current HDT* source code has not been published, the experiments here involved careful implementation based on the source code of the first version of HDT*~\cite{Qi_2016}. However, ideal implementation was challenging because HDT* uses its own Siamese networks to evaluate similarity between the template and the predicted target bounding box, and it implicitly assumes that expert predictions are based on the same bounding box size. This assumption conflicts with the study's experimental setup using arbitrary experts. To make the comparison as fair as possible in light of the above issues, the same criteria were used for similarity evaluation. Specifically, $V_T()$ from (\ref{eq:offline}) was used rather than the Siamese network in HDT* implementation. In addition, as the same experts were used for HDT*, AUC-based evaluation for HDT* was omitted due to related influence from bounding box sizes. \par
HDT* evaluates experts by the following two criteria given for every frame. The first criterion is the difference in appearance between the template image of the target and the cropped image for the predicted location by an expert. The second is the location difference between expert prediction and feedback (i.e., the weighted average of expert prediction). Reliable evaluation using the first criterion is challenging when the target may be occluded and/or shows heavy deformation. Moreover, since the feedback in the second criterion is directly determined via expert prediction (rather than from more reliable information), it is often unreliable in tough tracking situations. \par
Fig.~\ref{fig:hdt_compare} shows two typical cases of HDT* failure. Since feedback is the weighted average of expert prediction (rather than another more reliable pseudo-ground-truth, like input from an offline tracker), performance will degrade when the weight of the wrong expert is high. In addition, the final location is determined via Hedge's prediction, and is sometimes distant from all other expert predictions.

\section{Offline tracker details\label{sec:append_offline}}
In the study's experiments, an offline tracker based on Dijkstra's algorithm was used although any accurate offline tracker providing the delayed feedback can be applied. First, a graph linking two consecutive anchor frames, $u_q$ and $u_{q+1}$, was produced with node content corresponding to one of $f^{u_q} \cup\{f_i^\tau | \tau\in[u_q+1, u_{q+1}]\}$, where $f^{u_q}$ is the target location based on AAA at the previous anchor frame $u_q$, and $f_i^\tau$ is the target location based on the $i$th expert at $\tau$. The edges of the graph were assigned between nodes $f^{\tau-1}_i$ and $f^\tau_j$ with cost $C\left(f^{\tau-1}_i, f^\tau_j\right)$. According to~\cite{Li_Zhang_2008}, the cost was defined as
\begin{equation}
   C\left(f^{\tau-1}_i, f^\tau_j\right) = -\log \left(\mathcal{P}(f^{\tau-1}_i, f^\tau_j) V_E(f^{\tau-1}_i, f^\tau_j) 
   V_T(f^\tau_j)\right), \label{eq:offline}
 \end{equation}
where $\mathcal{P} \in [0,1]$ is the value of GIoU between two bounding boxes 
and $V_E \in [0,1]$ is the normalized cosine similarity between the feature vectors for the bounding box regions. The feature vectors are given by ResNet, as per the anchor frame determination in Sec.~\ref{sec:anchor}.
$V_T \in [0,1]$ is the normalized cosine similarity to the feature vector of the given template image.
The globally minimum cost path between $f^{u_q}$ and one of $\{f^{u_{q+1}}_i\}$ is determined using Dijkstra's algorithm, and the node sequence of the path gives the pseudo-ground-truth $y^\tau, \tau\in [u_{q-1}+1, u_q]$ as delayed feedback. 

\section{Derivation of Theorem \ref{theo:regret_delay}\label{sec:append_regret}}
Quanrud et al.\cite{Quanrud2015OnlineLW} gives the regret bound for the {\em Online Mirror Descent (OMD) algorithm} (a versatile choice for AEA) with delayed feedback. Various AEA algorithms can be derived from OMD by changing its regularizer for the expert selection process. Based on Theorem A.5 of \cite{Quanrud2015OnlineLW}, the regret bound of OMD with delayed feedback is given as
$$
R_T = O \left( \frac { 1 } { \eta }\psi + \eta \frac { G ^ { 2 } ( T + D ) } { \sigma } \right),
$$
where the values $\psi$ and $\sigma$ are determined by the regularizer choice, $\eta$ is the learning rate and $G$ is the difference between the maximum and minimum values of a loss function.\par
The authours' expert selection algorithm based on weights (\ref{eq:weight}) is also a special case of OMD. According to \cite{Quanrud2015OnlineLW}, if an entropic regularizer is applied, weight can be derived by updating the (\ref{eq:weight}) scheme. Moreover, $\sigma$ is 1 and $\psi$ is upper-bounded by $O(\ln N)$~\cite{Quanrud2015OnlineLW}.
The loss function (\ref{eq:loss}) takes values in $[0,1]$, and therefore $G=1$. Based on these values and the assumption that $\eta\propto\sqrt{\ln N/(T+D)}$, Theorem~\ref{theo:regret_delay} is supported.

\section{Determination of the anchor frame threshold \texorpdfstring{$\theta$}{}\label{sec:append_threshold}}
The hyper-parameter $\theta$, which is the threshold used to determine the anchor frame, was determined experimentally using the GOT10K dataset as noted in Sec.\ref{sec:threshold}.
Fig.~\ref{fig:theta} shows the AUC score and the anchor frame ratio based on changing the hyper-parameter $\theta$ from 0.6 to 0.9 at intervals of 0.1. Here, AUC is normalized to the range between $[0,1]$ for easier observation. As a general trend, a smaller threshold value results in a high anchor ratio and, in turn, a higher AUC. However, detailed observation reveals that this trend does not always hold because anchor frames determined using lower values of $\theta$
are less reliable and will result in inaccurate feedback. Consequently, smaller threshold values do not always give better performance.
Specifically, the best $\theta$ values for each expert group was as follows: \textit{High}:0.69; \textit{Low}:0.60; \textit{Mix}:0.65; \textit{SiamDW}:0.65; and \textit{SiamRPN++}:0.61.
\par
Evaluation was also performed to determine whether AAA can identify the target location with high confidence at anchor frames. AAA accuracy with experts in the \textit{High} group is shown in Table~\ref{table:high_anchor}, in addition to the accuracy of individual experts. The hyper-parameter $\theta$ was set at 0.69, as indicated by Fig.~\ref{fig:theta}. AAA achieved the best AUC score for all datasets. This proves that the proposed method can be applied to determine anchor frames appropriately.

\section{Proof of Proposition \ref{prop:ours}\label{sec:append_prop}}
\begin{proof}
If tracking performance evaluation is based on loss $\ell$ (rather than AUC or DP), the situation in which AAA outperforms experts on average over $\mathcal{V}$ is expressed as the inequality
\begin{equation}
\sum_{v\in \mathcal{V}}\sum_{t=1}^T\ell(p^{v,t}, y^{v,t}) \leq 
\sum_{v\in \mathcal{V}}\sum_{t=1}^T\ell( f_{i^*}^{v,t}, y^{v,t}),  
\label{eq:average-inequality}
\end{equation}
where $p^{v,t}$ is the result of location identification using AAA at $t$ for the sequence $v\in V$ and $y^{v,t}$ is the pseudo-ground-truth given as delayed feedback. As defined in Sec.~\ref{sec:average-performance}, the $i^*$th expert is the overall-best expert having the most optimal average performance over $\mathcal{V}$. 
The left side of (\ref{eq:average-inequality}) is relative to the loss of AAA over $\mathcal{V}$. The right side is relative to the loss of the  overall-best expert over $\mathcal{V}$, as defined by (\ref{eq:overall-best}). 
Accordingly, the inequality (\ref{eq:average-inequality}) indicates the situation in which AAA outperforms all experts on average over $\mathcal{V}$.\par
From (\ref{eq:average-inequality}), the following inequality is derived: 
\begin{eqnarray}
\sum_{v\in \mathcal{V}}\left[\sum_{t=1}^T\ell(p^{v,t}, y^{v,t})-\min_i\sum_{t=1}^T\ell(f_i^{v,t}, y^{v,t})\right]\nonumber \\
\leq 
\sum_{v\in \mathcal{V}}\left[\sum_{t=1}^T\ell( f_{i^*}^{v,t} y^{v,t})-\min_i\sum_{t=1}^T\ell(f_i^{v,t}, y^{v,t})\right].\label{eq:pp0}
\end{eqnarray}
The left side is the sum of AAA regrets (see (\ref{eq:regret})) over $\mathcal{V}$, and is therefore equal to $|\mathcal{V}|\overline{R}^\mathcal{V}$. The right side of (\ref{eq:pp0}) can be decomposed into two terms by splitting $\mathcal{V}$ into $\mathcal{S}$ and $\mathcal{V}\setminus\mathcal{S}$. Then, the first term is
\begin{equation}
\sum_{v\in \mathcal{S}}\left[\sum_{t=1}^T\ell( f_{i^*}^{v,t}, y^{v,t})-\min_i\sum_{t=1}^T\ell(f_i^{v,t}, y^{v,t})\right]=0, \label{eq:pp1}
\end{equation}
because the overall-best (i.e., the $i^*$th) expert is the best expert in the image sequence $v\in\mathcal{S}$.
By the definition of $\delta$ in (\ref{eq:prop_cond}), the second term is
\begin{equation}
\sum_{v\in \mathcal{V}\setminus\mathcal{S}}\left[\sum_{t=1}^T\ell( f_{i^*}^{v,t}, y^{v,t})-\min_i\sum_{t=1}^T\ell(f_i^{v,t}, y^{v,t})\right] \geq |\mathcal{V}\setminus\mathcal{S}|\delta.
\label{eq:pp2}
\end{equation}
From (\ref{eq:pp1}) and (\ref{eq:pp2}), we have: \begin{equation}
|\mathcal{V}\setminus\mathcal{S}|\delta \leq \sum_{v\in \mathcal{V}}\left[\sum_{t=1}^T\ell( f_{i^*}^{v,t}, y^{v,t})-\min_i\sum_{t=1}^T\ell(f_i^{v,t}, y^{v,t})\right].\label{eq:ppR}
\end{equation}
\par
From (\ref{eq:pp0}) and (\ref{eq:ppR}), {\em if the condition} $|\mathcal{V}|\overline{R}^\mathcal{V}\leq$ $|\mathcal{V}\setminus\mathcal{S}|\delta$ {\em is satisfied}, this gives \begin{eqnarray*}
|\mathcal{V}|\overline{R}^\mathcal{V}
&=&
\sum_{v\in \mathcal{V}}\left[\sum_{t=1}^T\ell(p^{v,t}, y^{v,t})-\min_i\sum_{t=1}^T\ell(f_i^{v,t}, y^{v,t})\right] \\
&\leq& |\mathcal{V}\setminus\mathcal{S}|\delta \\
&\leq& \sum_{v\in \mathcal{V}}\left[\sum_{t=1}^T\ell( f_{i^*}^{v,t} y^{v,t})-\min_i\sum_{t=1}^T\ell(f_i^{v,t}, y^{v,t})\right].
\end{eqnarray*}
This means that if the condition holds, 
(\ref{eq:pp0}) holds and immediately
(\ref{eq:average-inequality}) holds. Consequently, AAA outperforms other experts on average over $\mathcal{V}$ if $|\mathcal{V}|\overline{R}^\mathcal{V}\leq |\mathcal{V}\setminus\mathcal{S}|\delta$, and  Proposition \ref{prop:ours} is derived.
\end{proof}

\ifCLASSOPTIONcompsoc
  \section*{Acknowledgments}
\else
  \section*{Acknowledgment}
\fi

This work was supported by JSPS KAKENHI Grant Number JP17H06100 and JP18K18001.

\ifCLASSOPTIONcaptionsoff
  \newpage
\fi

\bibliographystyle{IEEEtran}
\bibliography{citation}

\begin{IEEEbiographynophoto}{Heon Song}
received B.E. degree from Kyushu University in 2019. He joined Graduate School of Information Science and Electrical Engineering at Kyushu University in 2019, where he is currently a M.S. candidate under the supervision of Prof. Seiichi Uchida. His research interests include object tracking and machine learning.
\end{IEEEbiographynophoto}

\begin{IEEEbiographynophoto}{Daiiki Suehiro}
received Ph.D. from Kyushu University in 2014. Currently, he is an assistant professor at Department of Advanced Information Technology, Information Science and Electrical Engineering in Kyushu University and he is also the team member of Computational Learning Theory Team, RIKEN Center for Advanced Intelligence Project. His research interests include machine machine learning theory, data mining, and time-series analysis.
\end{IEEEbiographynophoto}

\begin{IEEEbiographynophoto}{Seiichi Uchida}
received B.E. and M.E. and Dr. Eng. degrees from Kyushu University in 1990, 1992 and 1999, respectively. From 1992 to 1996, he joined SECOM Co., Ltd., Japan. Currently, he is a distinguished professor at Kyushu University. His research interests include pattern recognition and image processing. He received 2007 IAPR/ICDAR Best Paper Award, 2010, ICFHR Best Paper Award, and many domestic awards.
\end{IEEEbiographynophoto}

\end{document}